\documentclass{article}

\usepackage{microtype}
\usepackage{graphicx}
\usepackage{subfigure}
\usepackage{multirow}
\usepackage{booktabs} % for professional tableshttps://www.overleaf.com/project/6537fb3d3c25d7e1cbe6ae2b

\usepackage{hyperref}
\usepackage{url}
\usepackage{comment}
\usepackage{verbatim}

\usepackage[accepted]{icml2024}

\usepackage{amsmath}
\usepackage{amssymb}
\usepackage{mathtools}
\usepackage{amsthm}
\usepackage{dcolumn}
\usepackage{natbib}
\setcitestyle{numbers,comma,sort&compress,square}
\theoremstyle{plain}
\newtheorem{theorem}{Theorem}[section]

\newtheorem{lemma}[theorem]{Lemma}
\newtheorem{corollary}[theorem]{Corollary}
\theoremstyle{definition}
\newtheorem{definition}[theorem]{Definition}

\newtheorem{example}[theorem]{Example}
\theoremstyle{remark}

\DeclareMathOperator{\trace}{tr}
\DeclareMathOperator{\var}{Var}
\DeclareMathOperator{\sub}{Sub}
\DeclareMathOperator{\sign}{sgn}
\DeclareMathOperator{\argmin}{argmin}
\DeclareMathOperator{\argmax}{argmax}
\DeclareMathOperator{\irr}{Irr}

\icmltitlerunning{Grokking Group Multiplication with Cosets}

\begin{document}

\twocolumn[
\icmltitle{Grokking Group Multiplication with Cosets}



\icmlsetsymbol{equal}{*}

\begin{icmlauthorlist}
\icmlauthor{Dashiell Stander}{1}
\icmlauthor{Qinan Yu}{1,2}
\icmlauthor{Honglu Fan}{1,3}
\icmlauthor{Stella Biderman}{1}
\end{icmlauthorlist}

\icmlaffiliation{1}{EleutherAI}
\icmlaffiliation{2}{Brown University}
\icmlaffiliation{3}{University of Geneva}

\icmlcorrespondingauthor{Dashiell Stander}{dash.stander@gmail.com}

\icmlkeywords{Machine Learning, Grokking, Interpretability, Group Theory, Harmonic Analysis}

\vskip 0.3in
]



\printAffiliationsAndNotice{} % otherwise use the standard text.

\begin{abstract}
The complex and unpredictable nature of deep neural networks prevents their safe use in many high-stakes applications. There have been many techniques developed to interpret deep neural networks, but all have substantial limitations. Algorithmic tasks have proven to be a fruitful test ground for interpreting a neural network end-to-end. Building on previous work, we completely reverse engineer fully connected one-hidden layer networks that have ``grokked'' the arithmetic of the permutation groups $S_5$ and $S_6$. The models discover the true subgroup structure of the full group and converge on neural circuits that decompose the group arithmetic using the permutation group's subgroups. We relate how we reverse engineered the model's mechanisms and confirmed our theory was a faithful description of the circuit's functionality. We also draw attention to current challenges in conducting interpretability research by comparing our work to \citet{chughtai_toy_2023} which alleges to find a different algorithm for this same problem.
\end{abstract}

\section{Introduction}
\label{introduction}

Many methods have been proposed to render deep neural networks \emph{interpretable}. There is both an academic interest in understanding how neural networks do what they do and a societal interest in ensuring that decisions made by such models are sound, unbiased, and subject to human review. These concerns are not new, nor are they unique to deep neural networks. Many of the techniques developed (such as SHAP values \cite{lundberg2017unified}, saliency maps \cite{simonyan2014deep}, gradient attribution \cite{shrikumar2017just}, dimension reduction \cite{wold1987principal}, etc...) are still widely used today, but there is an understanding that such methods must be used as just one part of a careful analysis. Naive applications of even the most sophisticated algorithms will give misleading results \citep{adebayo2020sanity, bolukbasi2021interpretability, doshivelez2017rigorous, Jain2019AttentionIN}.

Mechanistic interpretability seeks to find ``neural circuits'' within deep neural networks, small sub-networks that act as connected computation graphs and accomplish a task. In ``toy'' (highly constrained) settings mechanistic interpretability has been successful, with multiple examples where the inner workings of neural networks have been successfully reverse engineered end-to-end \citep{gromov_grokking2023, nanda_progress_2023, nanda_emergent_2023, quirke2024understanding, transformers_recursion2023}. There have also been encouraging early successes in finding interpretable circuits within real-world models \citep{geva2021transformer, lieberum2023chinchilla, mcgrath2023hydra, olsson_-context_2022, interp_in_wild2022}, but there is already work emerging that illustrates how neural networks can resist common ``mechanistic interpretability'' methods \citep{friedman2023interpretability, makelov2023subspace, wen2023transformers}.

The toy interpretability projects that have succeeded have done so in large part because a distinct ground truth circuit that encodes the true nature of the task or environment emerged in the model. We build on this tradition and study a model that has perfectly learned to multiply permutations of five and six elements, which in mathematics is known as the symmetric groups $S_5$ and $S_6$, which are deeply studied and well-understood objects \citep{diaconis_persi_group_1988, dummit_abstract_2003, fulton_representation_1991}. We succeed in completely reverse engineering the model and enumerating the diverse circuits that it converges on to implement the multiplication of the symmetric group. Our work does not, however, represent an unmitigated success for the project of mechanistic interpretability. The prior work of \citet{chughtai_toy_2023} studied the exact same model and setting, but came to completely different conclusions. Understanding why our and \citet{chughtai_toy_2023}'s interpretations of the same data diverged required extensive effort (see Appendix \ref{sec:toy-model-universality} for a thorough comparison). \textbf{We find that even in a setting as simple and well understood as group arithmetic, it is incredibly difficult to do interpretability research and be confident about one's conclusions.}

Our main contributions are as follows:
\begin{itemize}
    \itemsep0em 
    \item We completely reverse engineer a one-hidden layer fully-connected network trained on the permutation groups $S_5$ and $S_6$.
    %\item We analyze the diversity of circuits that a single model architecture will form to solve a task with different random seeds. We also analyze how the distribution of circuits changes as the model is made wider.
    %\item We present a novel application of the Group Fourier Transform to circuit-level mechanistic interpretability.
    \item We apply a methodology inspired by \citet{geiger2023causal} to use causal experiments to thoroughly test all of the properties of our proposed circuit.
    \item We survey current research in mechanistic interpretability and draw connections between the difficulty of our work and broader challenges in the field.
\end{itemize}

\section{Related Work}

\textbf{Mechanistic Interpretability} Interpreting and reverse engineering the mechanism used to complete a given task is an active field in interpretability. Analysis of such mechanisms and circuits are discovered mainly through a top-down approach of causal mediation analysis. In the previous work \citet{hanna_how_2023, meng2023locating, tigges2023linear, wang2022interpretability}, the circuits are composed at the ``component level'' using the feed-forward layer and attention heads. We analyze the mechanisms of neural networks at the \emph{circuit level} of individual and small groups of neurons, drawing directly on the work of \citet{nanda_progress_2023, nanda_emergent_2023, olah_zoom_2020, quirke2024understanding, clock_and_pizza2023, transformers_recursion2023}. Our work builds directly on ``A Toy Model of Universality'' by \citet{chughtai_toy_2023}. We recreated precisely their experimental setup for the groups $S_5$ and $S_6$, though we came to different conclusions.

\textbf{Grokking} The models we study exhibit ``grokking'', wherein the model first memorizes the training set and then much later generalizes to the held out data perfectly. Grokking was first identified by \citet{power_grokking_2022} and has been well studied for its counter-intuitive training dynamics \citep{kumar_grokking_2023, grokking_competition_merrill2023, grokking_efficiency2023, grokking_phase_transition2023, benign_grokking2023, wei2022emergent}. We conducted all the analysis on fully grokked models with perfect test accuracy, as models that show this behavior have often formed clean generalizing circuits that are more easily interpreted \citep{gromov_grokking2023, nanda_progress_2023}.

\textbf{Group Theory} We used many of the tools of group theory for our analysis, in particular the well-developed representation theory of the symmetric group. Tools for analyzing data on groups are well-laid out in \citet{clausen_fast_1993, cohen_group_2016, diaconis_persi_group_1988, kondor_risi_group_2008, kondor_generalization_2018, huang_fourier_2009, karjol_neural_2023, plumb_snfft_2015}.

\section{Mathematical Preliminaries}

\begin{figure}
    \centering
    \includegraphics[width=5cm]{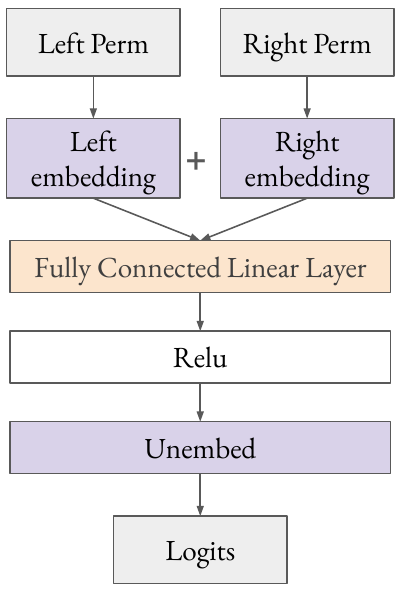}
    \caption{Model Architecture: we follow the model architecture used by \citet{chughtai_toy_2023}. The one-hot vectors of left and right permutations pass through separate embeddings. We concatenate the embeddings and pass them through a single fully-connected hidden layer with \texttt{ReLU} activations. An unembedding matrix transforms the activations into logits.}
    \label{fig:model_archiecture}
\end{figure}

This paper requires a familiarity with functions on groups, a topic that is uncommon in machine learning research. In this section we give an overview of the major concepts as they are realized in the permutation groups that we study. For a more formal introduction to group theory, please refer to Appendix \ref{appendix:group-theory}.

\subsection{Permutations and the Symmetric Group}
\label{sec:permutations-basics}

A permutation of $n$ elements is a map $\sigma$ that sends one ordering of $n$ elements to a different ordering. For example the order-reversing permutation on four elements would be: \[
(1 \; 2 \; 3 \; 4) \stackrel{\sigma}{\mapsto} (4 \; 3 \; 2 \; 1)
\]

The identity permutation, denoted $e$, leaves the ordering unchanged:

\[
(1 \; 2 \; 3 \; 4) \stackrel{e}{\mapsto} (1 \; 2 \; 3 \; 4)
\]

We refer to specific permutations by identifying them with the image of their action on the elements $[n] \coloneqq \{1, 2, \dots, n\}$ in increasing order. For the above example we would simply denote the order reversing permutation on four elements as $(4 \; 3 \; 2 \; 1)$.

We multiply two permutations on $n$ elements $\sigma, \tau$ by composition, read from right to left. If $\sigma = (4 \; 3 \; 2 \; 1)$ and $\tau = (3 \; 2 \; 1 \; 4)$, then 
$\sigma\tau$ is the permutation we obtain by first applying $\tau$ and then applying $\sigma$ to the output of $\tau$:
\[(1 \; 2 \; 3 \; 4)  \stackrel{\tau}{\mapsto} (3 \; 2 \; 1 \; 4)  \stackrel{\sigma}{\mapsto} (4 \; 1 \; 2 \; 3) \]

First applying $\tau$ and then $\sigma$ has the same effect as just applying the permutation $(4 \; 1 \; 2 \; 3)$. Additionally every permutation $\sigma$ has an inverse $\sigma^{-1}$ such that $\sigma\sigma^{-1} = e$. These properties makes all of the permutations on $n$ elements a \emph{group} called the symmetric group, which we write $S_{n}$.

There are six permutations in $S_4$ that do not change the position of $4$:
\[\begin{matrix}
 (1 \; 2 \; 3 \; 4) & (2 \; 1 \; 3 \; 4) & (3 \; 2 \; 1 \; 4) \\
 (1 \; 3 \; 2 \; 4) & (3 \; 1 \; 2 \; 4) & (2 \; 3 \; 1 \; 4)
\end{matrix}\]

These six permutations form a \emph{subgroup} of $S_4$ because multiplication is closed within that subset, multiplying any two permutations that leave $4$ unchanged results in another permutation that leaves $4$ unchanged. You can see that these six permutations are isomorphic to $S_3$ by simply ``forgetting'' about the $4$ that is fixed in the fourth position. In the paper, we will refer to the subgroup of $S_n$ isomorphic to $S_{n-1}$ that leaves element $i$ fixed as $H_i$.

One of the simplest types of permutations is a ``transposition,'' a permutation $\tau \in S_n$ that switches (``transposes'') two elements $i,j \in [n]$ and leaves the remaining elements fixed. Every element of $S_n$ can be decomposed into a product of transpositions. A given decomposition of a permutation is not unique, but the number of transpositions in the decomposition is an invariant of the permutation. For a permutation $g \in S_n$ if a set of transpositions $\tau_{1}\tau_{2}\dots\tau_{k} = g$, then every possible such set of transpositions will also have $k$ elements. The permutations that have an even number of transpositions are referred to as ``even'' permutations and those with an odd number are ``odd.'' The set of all even permutations in $S_n$ is a subgroup referred to as the ``alternating group'' $A_n$.

If we take $H_4 < S_4$ and multiply every element of on the left by some element $\sigma \in S_4$ then we get a \emph{left coset} of $H_4$ denoted $\sigma H_4$. The transposition $\tau = (4 \; 2 \; 3 \; 1)$ switches the elements in the first and fourth positions. The elements of $\tau H_4$ are:

\[\begin{matrix}
 (4 \; 2 \; 3 \; 1) & (4 \; 1 \; 3 \; 2) & (4 \; 2 \; 1 \; 3) \\
 (4 \; 3 \; 2 \; 1) & (4 \; 1 \; 2 \; 3) & (4 \; 3 \; 1 \; 2)
\end{matrix}\]

This coset is characterized by every element having $4$ in the first position. Every element of $H_4$ has $4$ in the fourth position and $\tau$ switches the first and fourth positions. For any $h \in H_4$, $h\tau$ has $4$ in the first position because $\tau$ moves it from the fourth. We would get a coset with all of the elements of $S_4$ with $4$ in the third position if we multiplied $H_4$ on the left by the any permutation that switches three and four.

There are also \emph{right cosets} where every element in a subgroup is multiplied from the right. The elements of $H_{4}\tau$ are:

\[\begin{matrix}
 (4 \; 2 \; 3 \; 1) & (1 \; 4 \; 3 \; 1) & (3 \; 2 \; 4 \; 1) \\
 (4 \; 3 \; 2 \; 1) & (3 \; 4 \; 2 \; 1) & (2 \; 3 \; 4 \; 1)
\end{matrix}\]

This right coset is characterized by every element having $1$ in the fourth position.

There are in fact four subgroups $H_i < S_4$ that are isomorphic to $S_3$, one where each element $\{1, \dots, 4 \}$ is fixed. In general there are at least $n$ subgroups of $S_n$ that are isomorphic to $S_{n-1}$. Any two $H_{i},\; H_{j}$ are \emph{conjugate} to each other. Conjugation by an element $\sigma$ maps $x \mapsto \sigma x \sigma^{-1}$. So if we have $H_4$ and conjugate it by $\sigma = (1 \; 4 \; 3 \; 2)$, then $\sigma H_4 \sigma^{-1}$ is $H_2$:

\[\begin{matrix}
 (1 \; 2 \; 3 \; 4) & (3 \; 2 \; 1 \; 4) & (4 \; 2 \; 3 \; 1) \\
 (1 \; 2 \; 4 \; 3) & (4 \; 2 \; 1 \; 3) & (3 \; 2 \; 4 \; 1)
\end{matrix}\]

If a subgroup is invariant to conjugation it is a \emph{normal} subgroup. The only normal subgroup of $S_n$ for $n > 4$ is the \emph{alternating group} $A_n$ of even permutations.

We will mostly refer to groups by name, but we will denote a general group as capital $G$ and a general subgroup as $H \le G$. For a proper subgroup ($H\neq G$), we will write $H < G$. For a normal subgroup, we will use $N \trianglelefteq G$.

\subsection{Fourier Transform over Groups}

Though Group Fourier Transform is not central to our presentation of the coset circuit, it was an important tool that we used to analyze the the activations of the trained models. It is also a critical part of \citep{chughtai_toy_2023}. We introduce the concepts here and go over the the similarities and differences between our work and \citep{chughtai_toy_2023} in Section \ref{sec:toy-model-universality}.

We begin with a presentation of the Discrete Fourier Transform (DFT), and then present the Group Fourier Transform by analogy. The DFT converts a function $f$  defined on $\{0, 1, \dots, n-1\}$ to a complex-valued function via the formula: \[ \hat{f}(k) = \sum_{t=0}^{n-1}f(t)e^{-2i\pi kt/n}, \quad k \in \{0, \dots, n-1\} \]
The DFT is commonly interpreted as a conversion from the \emph{time} domain to the \emph{frequency} domain because the $e^{-2i\pi kt/n}$ terms define a complex sinusoid with frequency $2\pi kt/n$. The frequency domain in this case means that these frequencies provide an alternative orthonormal basis from which we can work with functions. A function on $\{0, 1, \dots, n-1\}$ can be represented as a vector 
$f = \begin{pmatrix}
    x_0 & x_1 & \dots & x_{n-1} 
\end{pmatrix}^{\top}$
and its basis is given by the identity matrix $I_n$. The DFT defines a basis transformation, much like any other. The Fourier basis is given $n$ vectors. The first basis vector, corresponding to $k = 0$, is all ones. The $k=1$ basis vector is $\begin{pmatrix}
    1 & e^{-2i\pi/n} & \dots &e^{-2i\pi(n-1)/n}
\end{pmatrix}$, and all of the rest for up to $n\!-\!1$ are given by $\begin{pmatrix}
    1 & e^{-2i\pi k/n} & \dots &e^{-2i\pi(n-1)k/n}
\end{pmatrix}$.

%This basis has a number of useful properties, that stem from how well-fit it is to the structure of the domain. There are many functions that appear complicated in the time domain that are relatively simple or sparse in the frequency domain.

The DFT has a particularly nice interpretation as a function on the cyclic group $C_n$, which is isomorphic to addition modulo $n$. Please refer to Appendix \ref{appendix:group-theory} or to references such as \citep{diaconis_persi_group_1988, fulton_representation_1991, kondor_risi_group_2008} for a more detailed discussion.

The interpretation of the DFT as being over the cyclic groups can be generalized to non-commutative groups. We go over the construction in Appendix \ref{appendix:rep-theory} and \ref{appendix:group-fourier-transform}. The high level interpretation, however, is the same. For functions from $S_n \rightarrow \mathbb{C}$ there is an orthonormal basis that is equivariant to translations and convolutions. The frequencies for the Fourier transform over $S_n$ are given by the partitions of $n$. The ``highest'' frequencies can be interpreted as representing functions that are constant on permutations that all agree on a small number of elements of $[n]$ \citep{ellis_intersecting_2017}.

\section{Model Architecture} 
\label{sec: architecture}

As shown in Figure \ref{fig:model_archiecture}, the model we study contains separate left and right embeddings, followed by a fully connected linear layer with \texttt{ReLU} activations, and an unembedding layer. We use the same architecture as in \citep{chughtai_toy_2023} to enable consistent comparisons. \footnote{All code necessary for reproducing results and analysis is available at \url{https://www.github.com/dashstander/sn-grok}}

\begin{itemize}
    \item One hot vectors $\mathbf{x}_{g}$ with length $|G|$.
    \item Two embedding matrices, $\mathbf{E}_{l},\; \mathbf{E}_{r}$ with dimensions $(d, \; |G|)$, where $d$ is embedding dimension. $S_n$ is non-abelian, i.e. not commutative, and the separate embeddings are to give the model extra capacity.
    \item A linear layer $\mathbf{W}$ with dimension $(w, \; 2d)$, $w$ denoting the width of the linear layer. After the linear layer we apply the \texttt{ReLU} pointwise nonlinearity.
    \item  An unembedding layer $\mathbf{U}$ with dimension $(|G|, \; w)$, which transforms the outputs of the \texttt{ReLU} and linear layer to into logit space for the group.
\end{itemize}

We also note that the first $d$ columns of the linear layer will only act on the left embeddings and the second $d$ columns will only act on the right embeddings, so we can analyze $\mathbf{W}$ as the concatenation of two  $(w, \; d)$ matrices: $\mathbf{W} = [\mathbf{L} \; \mathbf{R}]$. 
\[
\mathbf{W}\begin{bmatrix}\mathbf{E}_{l}\mathbf{x}_{g} \\ \mathbf{E}_{r}\mathbf{x}_{h} \end{bmatrix} =  \mathbf{L}\mathbf{E}_{l}\mathbf{x}_{g} + \mathbf{R}\mathbf{E}_{r}\mathbf{x}_{h}
\]

Throughout the paper will refer to the values $\mathbf{L}\mathbf{E}_{l}\mathbf{x}_{g}$, $\mathbf{R}\mathbf{E}_{r}\mathbf{x}_{h}$, and their sum as ``\textbf{pre}-activations'' to denote that the \texttt{ReLU} activation function has not been applied. Post-\texttt{ReLU} values we refer to as ``activations.''

\section{Coset Circuits}
\label{sec: coset circuit}

\subsection{Sign Neurons Implement the Sign Circuit}
\label{sec:sign-circuit}

%\begin{figure*}
%    \centering   
%    \begin{subfigure}[Weight Analysis in Fourier Basis (Left)]{
%        \includegraphics[width=0.23\linewidth]%{figures/left_weight.pdf}
%        \label{fig:left_weights}}
%    \end{subfigure}
%    \hfill
%    \begin{subfigure}[Weight Analysis in Fourier Basis (Right)]{
%        \includegraphics[width=0.23\linewidth]%{figures/right_weight.pdf}
%        \label{fig:right_weights}}
%    \end{subfigure}
%    \hfill
%    \begin{subfigure}[Activations of all permutations combinations pre-relu]{
%        \includegraphics[width=0.23\linewidth]{figures/pre_relu.pdf}
%        \label{fig:pre_relu}}
%    \end{subfigure}
%    \hfill
%    \begin{subfigure}[Activations of all permutations combinations post-relu]{
%        \includegraphics[width=0.23\linewidth]{figures/post_relu.pdf}
%        \label{fig:post_relu}}
%    \end{subfigure}
%    \caption{A sign circuit at different stages of the linear layer. Figures \ref{fig:left_weigts} and \ref{fig:right_weights} show the pre-activations corresponding solely to the left and right inputs, respectively. Figure \ref{fig:pre_relu} shows the distribution of the full pre-activations, the sums of left and right. Figure \ref{fig:post_relu} shows the activations of the sign circuit, with the negative pre-activations in Figure \ref{fig:pre_relu} clipped to zero.}
%    \label{fig:sign-circuit}
%\end{figure*}

\begin{figure}
\begin{center}
    \includegraphics[width=6cm]{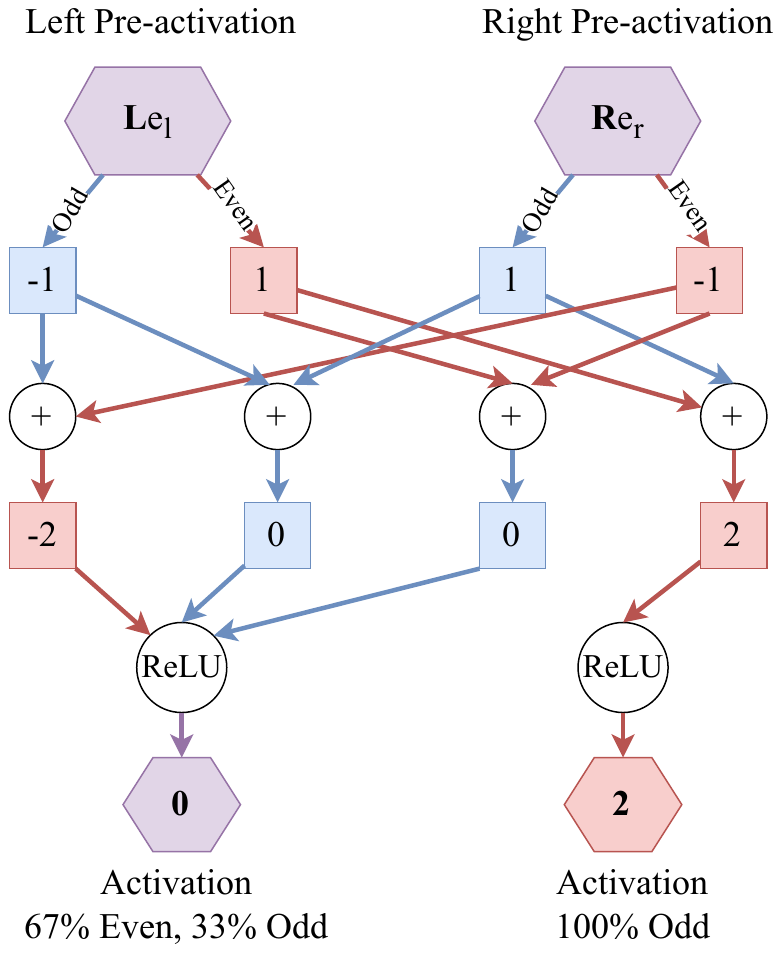}
    \caption{A diagram showing the four possible paths through a single neuron (i.e. one row of $\mathbf{R}\mathbf{E}_{r}$) that implements part of a ``sign circuit.'' The model stores whether a permutation is ``even'' or ``odd'' in the embeddings, represented in the left or right pre-activation values. The pre-activations are added together and then the \texttt{ReLU} activation is applied. The neuron only fires when the left permutation is even and the right is odd. If the neuron does \emph{not} fire, then in $1/3$ cases the product is odd and $2/3$ it is even.}
    \label{fig:sign-neuron}
    \end{center}
\end{figure}

The even permutations form a subgroup called the alternating group $A_n$. The two cosets of $A_n$ are the group itself and all of the odd permutations, $\tau A_n$. The multiplication of even and odd permutations has similar features to the addition of even and odd integers (hence the name). The sign map on a permutation in $S_n$, $\sign$, is given by: 
\[ \sign(\sigma) = \begin{cases} 1 \quad \sigma \in A_n \\ -1 \quad \sigma \in \tau A_n \end{cases} \]

An ``even'' permutation that is in $A_n$ is mapped to $1$ and an ``odd'' permutation \emph{not} in $A_n$ is mapped to $-1$. For any $\sigma,\rho \in S_n$, the sign of their product is the product of their signs: $\sign(\sigma \rho) = \sign(\sigma)\sign(\rho)$.

The one-layer model that we train uses this relationship to help solve the general group multiplication. Every single model we trained had at least two neurons dedicated to encoding the sign of the permutation product. Though the model cannot use the alternating group to completely solve multiplication in $S_n$, this \emph{sign circuit} is emblematic of the general coset circuits the model forms.

Consider the neuron shown in Fig. \ref{fig:sign-neuron}. The left pre-activations are given by $L(\sigma) = \sign(\sigma)$ and the right pre-activations are $R(\sigma) = -\sign(\sigma)$. The full action of the neuron is given by $\texttt{ReLU}(L(\sigma_{l}) + R(\sigma_{r}))$ and there are three cases: 

\begin{enumerate}
    \item \label{even-perm-zero} $\sign(\sigma_{l}) = \sign(\sigma_{r}) \Rightarrow \sign(\sigma_{l}\sigma_{r}) = 1$. In this case $L(\sigma_{l})$ and $R(\sigma_{r})$ destructively interfere, cancelling out to $0$. Both the pre-activation and activation are $0$.
    \item \label{odd-perm-pos} $\sign(\sigma_{l}) = -1, \;  \sign(\sigma_{r}) = 1 \Rightarrow \sign(\sigma_{l}\sigma_{r}) = -1 $. In this case $L(\sigma_{l})$ and $R(\sigma_{r})$ reinforce each other and sum to a positive value. Since $2 > 0$, the activation value is $2$. 
    \item \label{odd-perm-neg} $\sign(\sigma_{l}) = 1, \;  \sign(\sigma_{r}) = -1 \Rightarrow \sign(\sigma_{l}\sigma_{r}) = -1 $. Like in (\ref{odd-perm-pos}) the product $\sigma_{l}\sigma_{r}$ is an odd permutation and $L(\sigma_{l})$ and $R(\sigma_{r})$ constructively interfere, though this time $L(\sigma_{l}) + R(\sigma_{r}) = -2$, which is less than $0$. Thus \texttt{ReLU} clips the pre-activation and sends it to $0$.
    %so despite the information that $\sigma_{l}\sigma_{r}$ being even is present in the pre-activations, the activation that gets passed to the unembedding layer is identical to an even permutation as in (\ref{even-perm-zero}).
\end{enumerate}

\subsection{Conjugate Subgroup Circuit}
\label{sec:conjugate-subgroup-circuit}

All four ways to multiply two cosets of $A_n$ are well-defined. For each of the four options (even-even, odd-even, etc...) we know which coset of $A_n$ the product will be in, but no other subgroup of $S_n$ has this property. The model instead learns to use sets of conjugate subgroups. Recall that $H_i < S_n$ is the subgroup isomorphic to $S_{n-1}$ that fixes the element $i \in [n]$ in the $i^{\text{th}}$ place and $\tau_{ij}$ is the permutation that swaps $i$ and $j$. Any two $H_i$ and $H_j$ are conjugate to each other, $\tau_{ij}H_{i}\tau_{ij} = H_j$ and $\tau_{ij}H_{j}\tau_{ij} = H_i$. This means that there are two \emph{shared} cosets between $H_i$ and $H_j$, because $H_{i}\tau_{ij} = \tau_{ij}H_{j}$ and $H_{j}\tau_{ij} = \tau_{ij}H_i$. \textbf{The model implements the full group multiplication by picking out the shared cosets of conjugate subgroups.}

As an example, consider a neuron that corresponds to $H_{1}$ for the left permutation and $H_5$ for the right permutation. The shared coset is $H_{1}\tau_{15} = \tau_{15}H_{5}$, the set of all $\sigma \in S_5$ with $\sigma(1) = 5$. The pre-activations for the left and right permutations will be: 

\begin{align}
L(\sigma) = \begin{cases}
    4 \quad \sigma \in H_{1} \\
    2 \quad \sigma \in H_{1}\tau_{12} \\
    0 \quad \sigma \in H_{1}\tau_{13} \\
    -2 \quad \sigma \in H_{1}\tau_{14} \\
    -4 \quad \sigma \in H_{1}\tau_{15} \\
\end{cases} &
R(\sigma) = \begin{cases}
    -4 \quad \sigma \in \tau_{15}H_{5} \\
    -2 \quad \sigma \in \tau_{25}H_{5} \\
    0 \quad \sigma  \in \tau_{35}H_{5}\\
    2 \quad \sigma \in \tau_{45}H_{5} \\
    4 \quad \sigma \in H_{5} \\
\end{cases}
\end{align}

The final activation is still $\texttt{ReLU}(L(\sigma_{l}) + R(\sigma_{r}))$, but now there are twenty-five possible pairs of cosets. All twenty-five combinations can be boiled down to two meaningful cases:

\begin{enumerate}
    \item If $L(\sigma_{l}) + R(\sigma_r) = 0$, then $\sigma_{l}\sigma_{r}$ is \textbf{in} the shared coset $H_{1}\tau_{15}$.
    \item If $L(\sigma_{l}) + R(\sigma_r) \ne 0$, then $\sigma_{l}\sigma_{r}$ is \textbf{not in} the shared coset $H_{1}\tau_{15}$.
\end{enumerate}

Each left coset $yH_5$ has a paired right coset $H_{1}x$ such that $H_{1}xyH_{5} = H_{1}\tau_{15} = \tau_{15}H_{5}$. The discrete values that $L$ and $R$ can take are precisely tuned so that those pairs of left and right cosets cancel out. Just like with the sign neuron, information about the pre-activation being negative is lost with the \texttt{ReLU}. This lost information has to be made up with extra neurons that correspond to $(H_{1}, H_5)$ and assign different values to the cosets. For example, a different neuron that uses $-L(\sigma_{l}) - R(\sigma_{r})$ will fail to fire for a different set of permutations. The combination \[\texttt{ReLU}(L(\sigma_{l}) + R(\sigma_{r})) + \texttt{ReLU}(-L(\sigma_{l}) - R(\sigma_{r}))\] will be much closer to a perfect on/off switch for coset membership.

\subsection{Decoding Permutations with Coset Membership}

%\begin{figure}
%    \centering
%    \includegraphics[width=6cm]{figures/CosetBitSet.png}
%    \caption{An example of an $S_4 < S_5$ coset circuit bit pattern for the permutation $(5 \; 4 \; 3 \; 2 \; 1)$. The value of each position of the permutation is encoded in five coset circuits, twenty-five in total. Highlighted is the code for the third position. The value is $3$ because three is the only circuit for that position that did \emph{not} fire. Each individual circuit for a position-value pair (i.e. a single square in the diagram) will consist of as few as one or as many as fifteen neurons dedicated. The median number of neurons involved in a single coset circuit depends on the subgroup, but is around four.}
%    \label{fig:coset-bitset}
%\end{figure}

There are $n^2$ combinations of $(H_i, H_j)$ subgroups. Each pair can be interpreted directly as encoding the set of permutations with $i$ in the $j\textsuperscript{th}$ position. Because of the way the coset neurons function, each neuron is better understood as firing when the value in the $j\textsuperscript{th}$ position is certainly \emph{not} $i$. The $n^2$ combinations of $(H_i, H_j)$ uniquely identify each element of $S_n$. We can use the outputs of twenty-five $(H_i, H_j)$ neurons as a code that uniquely encodes each element of $S_5$. By analyzing the unembedding layer to see how the model makes use of $(H_i, H_j)$ neurons, we see that this is almost exactly what the model does. This same construction works for every subgroup of $S_n$ except for $A_n$.

%One of the most striking properties of the \emph{coset circuits} that the models form is that the \texttt{ReLU} activation function throws information away. In Fig. \ref{fig:pre_relu} we can see that the left and right pre-activations cancel out and are extremely close to zero if $\sign(\sigma_{l}) = \sign(\sigma_{r})$, i.e. when the product $\sigma_{l}\sigma_{r}$ is even. If the product $\sigma_{l}\sigma_{r}$ is odd then the pre-activations are $\pm 6$: $6$ if odd-even and $-6$ if even-odd. All of the information about the sign of the product exists in the pre-activation. If odd-even is even then the neuron's activation value of $6$ is propagated to the unembedding layer and the model receives an unambiguous signal that it should \emph{not} predict an even permutation. But the neuron does not fire for every odd permutation. If the input permutations are even-odd then the $-6$ pre-activation is clipped by the \texttt{ReLU}. The neuron firing means that the target permutation is odd, but the neuron \emph{not} firing is ambiguous.
\begin{figure*}[ht]
    \centering
    \includegraphics[width=\linewidth]{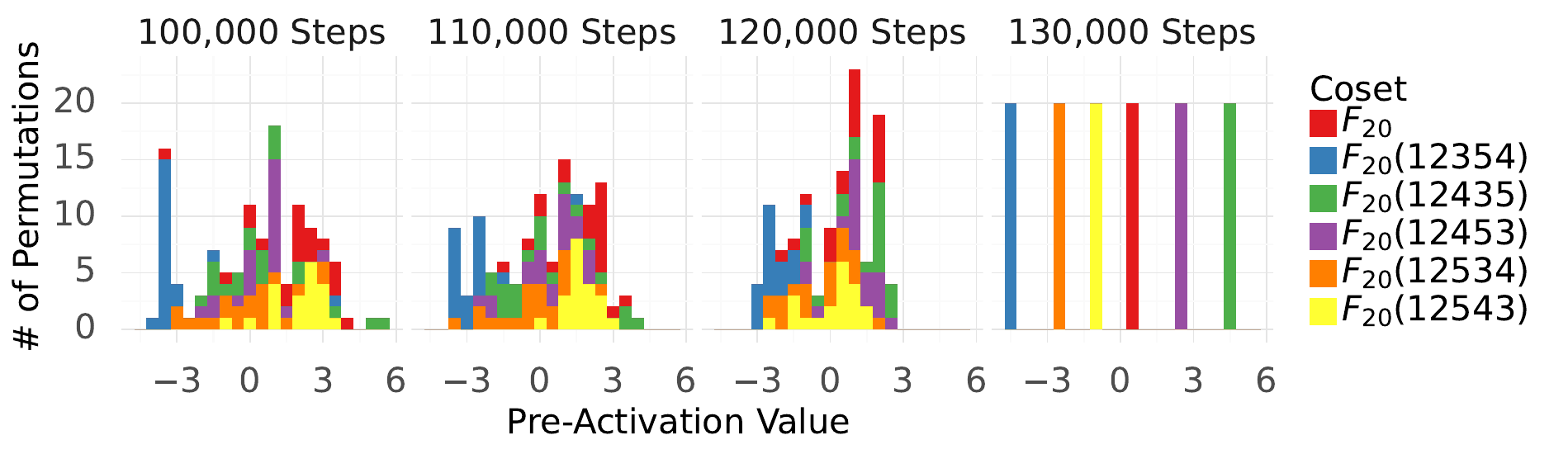}
    \caption{An illustration of the phenomenon of ``concentration on cosets,'' depicting the 115th neuron from seed 11. We show the evolution of the left pre-activations (the pre-\texttt{ReLU} outputs of a layer) of training on an $F_{20}$ neuron from 100k to 130k steps. The seed of the neuron's functionality is already present at 100k steps, where it fires very strongly and negatively for permutations in the coset $F_{20}(1 \; 2 \; 3 \; 5 \; 4)$, but it takes time for the action of the neuron to ``clean up'' on the other cosets of $F_{20}$. The distribution found at 130k steps does not change very much afterwards. Noticing this common pattern of neurons taking on these discrete values was a striking piece of evidence that required further investigation.} 
    \label{fig:coset-circuit-dev}
\end{figure*}

\section{The Process of Reverse Engineering}
\label{sec:reverse-engineering}
%In this section, we describe the process of discovering the coset circuit. We observe the activations and build up a theory for their mechanism. Second, we confirm the sufficiency and necessity of the circuits we have identified via ablations. Finally, we rigorously validate our hypothesis by testing the properties of the circuit through causal interchange interventions.

\subsection{Identifying Coset Circuits}
\label{sec:coset-concentration}

The first step in attempting to reverse engineer the mechanisms of a neural network is to spend some time staring at the weights and activations. Even a small one-layer model such as ours is too large to visualize all at once. It was not until we looked closely at the pre-\texttt{ReLU} activations that we produced a histogram similar to Figure \ref{fig:coset-circuit-dev}. The left and right pre-activations of one neuron were nearly constant on the distinct cosets of the Frobenius group of order 20 ($F_{20}$), one of the subgroups of $S_5$. \footnote{$F_{20}$ is equivalent to the group of affine transformations $x \mapsto ax + b$, where $a,b,x$ are in the field with five elements and $a \neq 0$.} Further investigation revealed that almost every neuron had this property of only producing a discrete number of values that corresponded directly to the cosets of one of the subgroups of $S_5$ or $S_6$. For a function $f: G \rightarrow \mathbb{R}$, we define $C_{H}(f)$ to be the degree to which $f$ concentrates on the cosets $H \le G$: \[ 
C_{H}(f) \coloneqq \frac{\sum_{gH}\var[f|_{gH}]}{\var[f]}
\]

Where $\var[f|_{gH}]$ is the variance of $f$ when the domain is restricted to the coset $gH$. Intuitively $C_{H}(f)$ calculates the degree to which restricting to the cosets of $H$ reduces the variance of $f$. If $C_{H}(f) < 1$ it implies that the activations $f$ can meaningfully be understood better by looking at the values that it takes on the cosets of some subgroup. Recall that a single neuron is a function $N_i: S_{n}\times S_{n} \rightarrow \mathbb{R}$ is the sum of two functions $G \rightarrow \mathbb{R}$, one for the left and right permutations, respectively. We can calculate $\min C_{H}$ for each. Take as an example $N^{l}_{115}$, the neuron shown in Figure \ref{fig:coset-circuit-dev}. At 100,000 steps (on the far left) $\var[N^{l}_{115}] = 5.23$. Its activations are not concentrated on the specific cosets of $F_{20}$, however, and $C_{F_{20}}(N^{l}_{115}) = 2.96$. At 130,000 steps (on the far right) $\var[N^{l}_{115}]$ has increased to $9.06$, but $C_{F_{20}}(N^{l}_{115}) < 10^{-5}$. The distribution within each coset of $F_{20}$ has close to zero variance.

We see a typical example of what this looks like for the entire model in Fig. \ref{figure:coset-concentration}. As the validation loss approaches a small value, there is a rapid transition from the median coset concentration being approximately $1$, to a minuscule value.

Even if it is apparent that a neuron is taking on discrete values and is a good candidate for being a coset neuron, it is difficult to tell by sight which subgroup the neuron is activating for. $S_5$ and $S_6$ only have 156 and 1,455 subgroups, respectively,  \footnote{Sequence A005432 OEIS \citep{oeis}} so it is tractable to do an exhaustive search and calculate 
\[
    \argmin_{H \in \sub(G)}C_{H}(f)
\]
the subgroup that minimizes the variance of $f$ for every neuron in the model. Running these calculations shows that for the 128 $S_5$ models and 100 $S_6$ models we trained over 99.2\% of the neurons in the linear layer had $\min_{H \in \sub(G)}C_{H}(f) < 1.0$, and the vast majority of those were less than $10^{-6}$.

With the ability to calculate directly which neurons corresponded to which subgroup, our theories for exactly what the neurons were representing fell into place. The next step was to confirm that these neurons were actually responsible for the models' performance.

\begin{figure*}[h!]
\centering
    \centering   
    \begin{subfigure}{
        \includegraphics[width=0.45\linewidth]{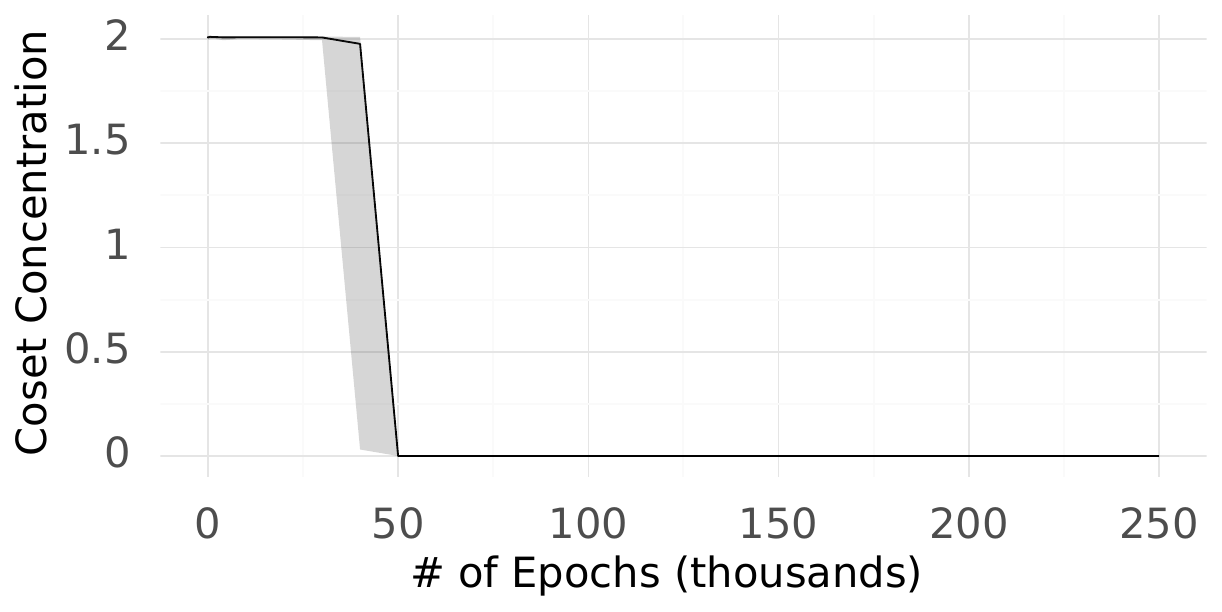}}    
    \end{subfigure}
    \begin{subfigure}{
        \includegraphics[width=0.45\linewidth]{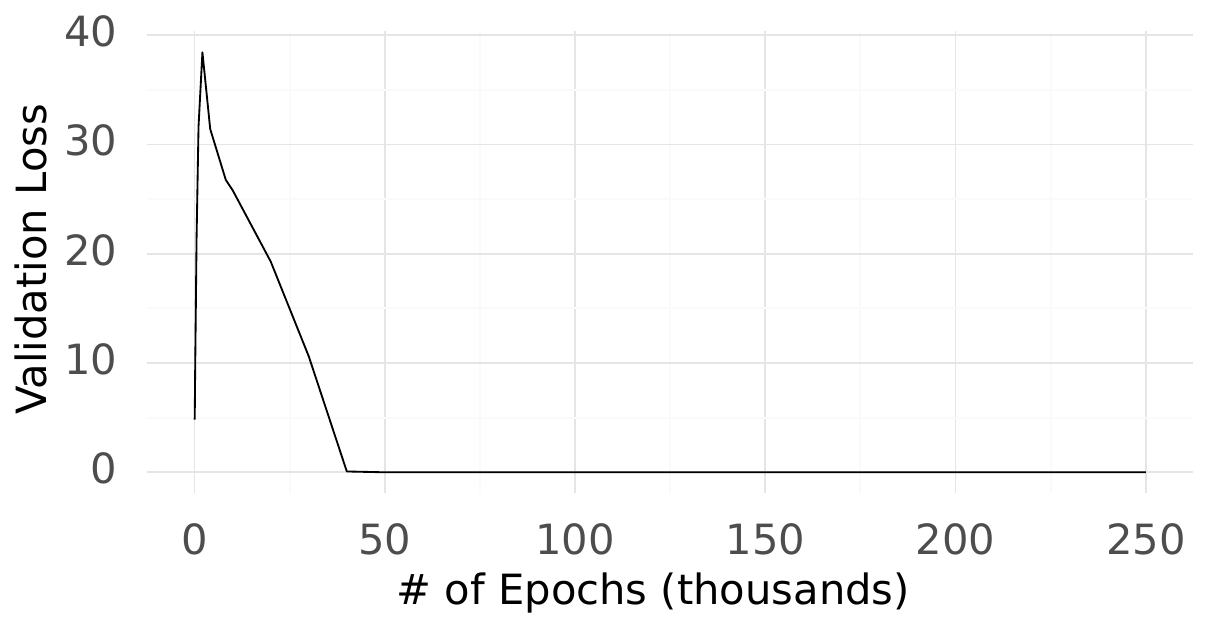}}
    \end{subfigure}
    \caption{The paired evolution of the the validation loss and $\min_{H \in Sub(H)} C_H$, which encodes the formation of coset circuits. Displayed is the $S_5$ model with random seed 1. Different runs will form coset circuits at different times in training, but the effect is representative.}   
    \label{figure:coset-concentration}
\end{figure*}

\begin{figure}[h]
    \centering
    \includegraphics[width=0.8\linewidth]{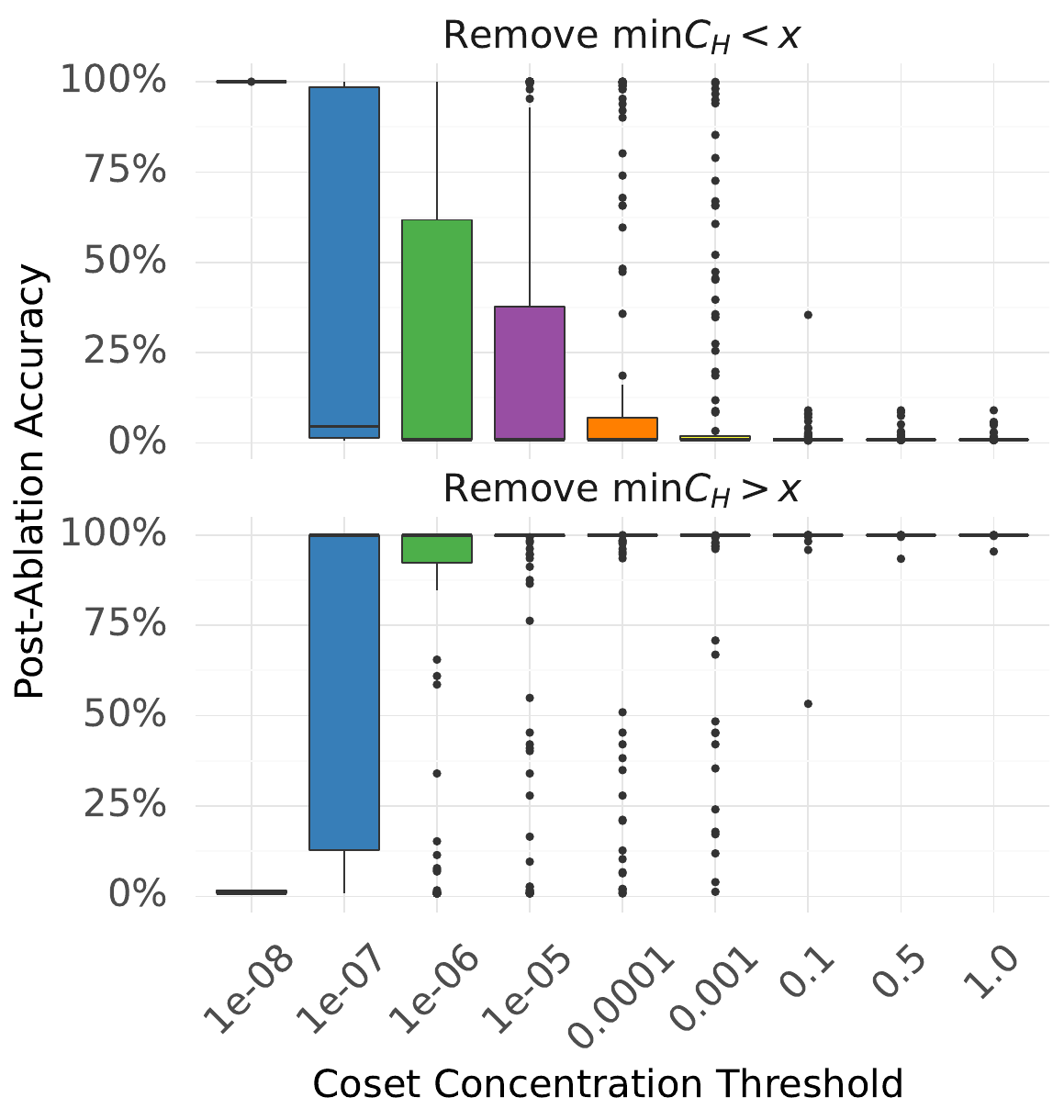}
    \caption{We perform ablations by re-calculating the accuracy after removing any neurons $N_i$ that have  $\min_{H \in \sub(G)}C_{H}(N_i)$ greater than (top figure) or less than (bottom figure) the thresholds on the x-axis.}
    \label{fig:ablations}
\end{figure}
\subsection{Ablations}
\label{sec:ablations}

We have described how coset neurons function and how they can be identified. We will now show via ablations that coset neurons are not solely sufficient but also necessary to implement multiplication in $S_n$. We conduct ablations by removing neurons which have a coset concentration $\min_{H \in \sub(G)}C_{H}(N_i)$ above a threshold.

If coset circuits are in fact responsible for the performance of our models, then we expect to see no change in the accuracy when the neurons that have not converged onto the cosets of a subgroup are removed from the model. This is precisely what we see on the far right of Figure \ref{fig:ablations}. Of the 128 $S_5$ models that we trained, 126 models saw no change in the accuracy when we removed the neurons with $\min_{H \in \sub(G)}C_{H}(N_i) \ge 1$ (the far right of Figure \ref{fig:ablations}). Recall that if $C_{H}(N_i) \ge 1$, restricting to the cosets of $H$ at best does not change the variance of $f$. Of the two models that did show a decrease in accuracy, they decreased to 99\% and 98\%.

We see more between-run variation when we remove more neurons. The median model has 24 out of 128 neurons with $\min C_{H}(N_i) \ge 10^{-5}$, but the $50^{\text{th}}$ and  $25^{\text{th}}$ percentile accuracy is still 100\%. It is not until we set the threshold to $10^{-6}$ that the $25^{\text{th}}$ percentile moves at all. When we set the threshold at $10^{-7}$ the performance for many models has collapsed, but the median model has had 42 neurons removed and the median accuracy is still 100\%. Recall also that the neuron shown in the far right of Figure \ref{fig:coset-circuit-dev} has a coset concentration of $10^{-5}$.

The overwhelming majority of neurons are identifiable as coset neurons. Of those neurons, those with the very highest concentration on cosets account for the largest portion of each model's performance.

\subsection{Causal Interventions}
\label{sec:causal-interventions}

\begin{table*}[htp]
\centering
\caption{Causal interventions aggregated over 128 runs on $S_5$ with different sizes}\label{tab:interventions}
\begin{tabular}{lccc}
\toprule
Intervention & Mean Accuracy & Mean Loss \\
\midrule
Base Model & 99.99\% & 1.97e-6  \\
Embedding Swap & 1\% & 4.76 \\
Switch Left and Right Sign & 100\% & 1.97e-6  \\
Switch Left Permutation Sign & 0\% & 22.39 \\
Switch Right Permutation Sign & 0\% & 22.36 \\
Perturb $\mathcal{N}(0, 0.1)$ & 99.99\% & 2.96e-6\\
Perturb $\mathcal{N}(0, 1)$ & 97.8\% & 0.0017  \\
Absolute Value Non-Linearity & 100\% & 3.69e-13 \\
Perturb $\mathcal{N}(1, 1)$ & 88\% & 0.029\\
Perturb $\mathcal{N}(-1, 1)$ & 98\% & 0.0021 \\
\bottomrule
\end{tabular}
\end{table*}

To rigorously test the properties of the coset circuit, we carefully designed \emph{causal} experiments to test specific properties of in the circuits. We observe a circuit's behavior over the entire data distribution (the full group $S_n$) and we see that our model of the circuit is consistent with the behavior of the true circuit. To confirm that our model of the circuit is correct, however, we need to ``break'' the circuit in targeted ways and test that it behaves in the way we predict. Neural circuits are complex enough that observational evidence is not enough. We aggregated runs over 128 $S_5$ models of different and recorded their average loss and accuracy. Initially, over the initial models without intervention, we have accuracy extremely close to 1.

\paragraph{Embedding Exchange} The left and right embeddings encode different information---membership in right and left cosets, respectively---and cannot be interchanged. To test this we intervene to switch the left and right embeddings. After the intervention, we observed a significant drop in accuracy to 0 and a rise in loss. This aligns with our expectation that the membership is an important property that can't be switched.

\paragraph{Switch Permutation Sign} The pre-activations are symmetric about the origin and the sign of the pre-activations does not matter, only whether or not the pre-activations is equal to zero. The \emph{relative} sign of the left and right pre-activations should matter a lot. To test this, we have three tests: changing the sign of just the left embeddings, just the right embeddings, and both embeddings. In the case where we change both the sign and with commutative property, we can still expect the left and right activation to cancel out. Therefore, we should see a near-perfect accuracy and near-0 loss. The result is as expected. When we change the sign of only the left or right embedding, such cancellation law doesn't hold anymore. Therefore, we observe a 0 accuracy in both cases. 

\paragraph{Absolute Value Non-linearity} The circuit can create a perfect 0-1 coset membership switch with multiple neurons on constructive interference, but every single neuron is noisy and fundamentally limited by the \texttt{ReLu} non-linearity. To test this, we replace the \texttt{ReLU} activation function with the absolute value function $x \mapsto |x|$. We observe perfect accuracy and an even lower loss that a half of the original loss. 

%\paragraph{\texttt{ReLU} Clip} If the pre-activations are greater than zero and aren't filtered by the \texttt{ReLU}, then the magnitude of the activations does not matter, there is no particular information in or distinction between a small or large pre-activation value. To test this, we patch the activations of neurons only when they are greater than some small constant and replace them with a normally distributed value centered around the mean activation of that neuron.

\paragraph{Distribution Change} It is essential to the functioning of each neuron that a large proportion of the pre-activations are close to zero. To test this we compare how adding noise from a $\mathcal{N}(-1, 1)$ and $\mathcal{N}(1, 1)$ affect the performance of the model. We can see that changing the distribution of the activation in Perturb $\mathcal{N}(-1, 1)$ changes the performance less significantly than  $\mathcal{N}(1, 1)$. This indicates that the coset requires 0 as a threshold value to decide the membership. 

The results of these interventions can be viewed in Table \ref{tab:interventions}

\section{The Group Composition via Representations Algorithm}
\label{sec:toy-model-universality}

Our experimental setup is identical to that of \citet{chughtai_toy_2023}, but our analysis led us to a different conclusion.\citet{chughtai_toy_2023} proposed the ``Group Composition via Representations'' (GCR) algorithm. They show that, given an irrep $\rho$ of $S_n$, $\argmax_{c \in S_n} \trace[\rho(a)\rho(b)\rho(^{-1}c)] = ab$ and propose that this is the algorithm the model is implementing. This requires that not only store the matrix irreps, but that the model \textit{perform the matrix multiplication} within its mechanism. We find that most of the evidence \citep{chughtai_toy_2023} put forward is also consistent with coset circuits. The other evidence we were not able to independently replicate. We also find evidence that, to our understanding, is not consistent with the GCR algorithm but is explained by coset circuits.

\subsection{Our Interpretation of the Evidence for GCR}

 \citet{chughtai_toy_2023} put forward four main pieces of evidence, which we restate here for clarity: (1) Correlation between the model's logits and characters of a learned representation $\rho$. (2) The embedding and unembedding layers function as a ``lookup table'' for the representations of the input elements $\rho(a),\; \rho(b)$ and the inverse of the target $\rho(c^{-1})$. (3) The neurons in the linear layer calculate the matrix product $\rho(a)\rho(b) = \rho(ab)$. (4) Ablations showing that the circuit they identify is responsible for the majority of the model's performance. Many of these points are equally consistent with the coset circuit and the other we could not find evidence for.

\textbf{Ablations} Though we do not perform all of the exact ablations that \citet{chughtai_toy_2023} perform, we also find that the weights that show high Fourier concentration and perform the coset multiplication are integral to the model's performance, see Section \ref{sec:ablations}.

\textbf{Irrep Look Up Table} We were not able to find any evidence that the embedding or unembedding layers function as a look-up table for any representation except for the one-dimensional sign representation.  We did find that the model's weights and activations \emph{concentrate} on specific irreps in the group Fourier basis. This is due, however, to concentration on cosets of specific subgroups, not because the matrix representations are realized anywhere in the weights. The relationship between functions that are constant on cosets and specific irreps is shown in Appendix \ref{constant_on_cosets}.

\textbf{Logit Attribution} The trace of a group representation is referred to as the ``character'' and often denoted $\chi$. We find that the model's logits correlate with the character $\chi_{\rho}(abc^{-1})$ when the irrep $\rho$ appears in the Fourier transform of the model's weights. This is not, however, because the model has implemented the matrix product $\rho(ab)\rho(c^{-1})$, but because the model is ``counting'' the number of cosets that $ab$ and $c$ are both in. We prove in \ref{counting_cosets}, if the cosets are of conjugate subgroups that have their Fourier transform concentrated on the irrep $\rho$ (as we observe for the models in question), then the number of shared cosets will also correlate with the characters of $\rho$.

\textbf{Matrix Multiplication of Irreps} We were not able to find any evidence that the linear layer implements matrix multiplication, again excluding scalar multiplication of the sign irrep.

\subsection{Evidence GCR Does Not Explain}
\label{appendix:evidence}
\textbf{Concentration on Cosets} In the standard basis the pre-activations of the overwhelming majority of neurons concentrate heavily on the cosets of subgroups. This is behavior is not predicted by the GCR algorithm. 

\textbf{The Difference Between Subgroups and Irreps} The GCR algorithm and coset circuit cannot be \textit{equivalent} because there is not, in fact, a one-to-one relationship between cosets and irreps. Most subgroups of $S_n$ have their Fourier transforms concentrate on more than a single group (see Table \ref{tab:s5_subgroups} for the spectral properties of all of the subgroups of $S_5$), indeed this needs to be the case as there are many more subgroups than irreps. Please refer to Table \ref{tab:sn_subgroups_sequence} for a concrete comparison and Appendix \ref{appendix:evidence} for an asymptotic analysis. We also observe coset circuits for some subgroups such as $D_{10}$\footnote{The dihedral group of order 10, the symmetry group of a pentagon.} will have coset circuits concentrated on both $(3,\;2)$ or $(2,\;2,\;1)$, depending on the run. The GCR algorithm would treat these as different circuits, though their behavior is in fact identical.

\textbf{Unembedding Correlations of Neurons} We observe that the correlation between in the unembedding of neurons that concentrate on the same coset is on average $81.4\%$ (see Table \ref{tab:unembed_corr}). The correlation between neurons concentrated only on the same conjugacy class of subgroup (e.g. $H_1$ and $H_2$) is on average $-0.2\%$. The neurons that represent subgroups in the same conjugacy class will oftentimes, though not always, be concentrated on the same irrep. The model is treating cosets together but the irreps and conjugacy classes separately.

\textbf{Coset Circuit Specific Causal Interventions} The property that the loss goes down when we replace the \texttt{ReLU} activation function with absolute value is a very strange property that GCR does not predict.

The concentration of the model's activations on irreps of $S_n$ is striking evidence and the GCR algorithm that \citep{chughtai_toy_2023} detail could indeed solve the problem of group multiplication. The coset circuit is also consistent with all of the evidence that \citep{chughtai_toy_2023} provide and is additionally consistent with evidence that the GCR algorithm does not explain.

\section{Discussion and Conclusion}
\label{sec:discussion}

We performed a circuit level analysis to discover the concrete mechanism a one layer fully connected network uses to solve group multiplication in $S_5$ and $S_6$. We showed that the model decomposes $S_5$ and $S_6$ into its cosets and uses this structural information to perfectly implement the task.

Though our work concerns a toy problem, we highlight a core takeaway that applies broadly to the field of interpretability: we must treat proposed neural mechanisms as theories until they have been thoroughly tested. 

When we identify what we believe to be a circuit within a larger network found via techniques such as \citep{conmy_towards_2023, goldowskydill2023localizing}, we have taken the first step towards mechanistically understanding how a model performs a task. The evidence we have for the circuit's role in that task is, however, fundamentally observational and correlational. The nodes in the circuit's computation graph are \textit{causally} connected, but the relationship between the action of those nodes is only \textit{observed} to be \textit{correlated} to a certain task with respect to a distribution. This is valuable information to have, but the understanding that it imparts is limited and must be recognized as such.

When beginning this project we quickly noticed that the activations of sub-circuits of our model were concentrated on specific irreps of $S_n$. It was only with additional investigation that we were able to attach semantic meaning to this phenomenon. We observed that the neurons concentrated on a single irrep were activating for specific subgroups. The hypothesis of the coset circuits had formed, but it was still only a theory. The facts we had observed were incontrovertible, but their reason was unclear. \textbf{It was only after performing the causal experiments detailed in Section \ref{sec:causal-interventions} that we became confident we understood the mechanism.} The simple reality is that more than one theory can be consistent with observational data, especially when that data only comes from a small subset of the full distribution. There is a long history of scholarship showing that interpretability techniques, including state-of-the-art, can give be misleading and contradictory results \citep{adebayo2020sanity, bolukbasi2021interpretability, casper2023red, doshivelez2017rigorous, friedman2023interpretability, hase2023does, Jain2019AttentionIN, makelov2023subspace, mcgrath2023hydra}.

In doing this work we had many advantages not available when interpreting real-world models: access to the entire distribution, an orthonormal basis for the function space of the network, and a relatively small model. The task of multiplication in $S_n$ is deterministic and very well studied, we had many mathematical tools to bring to bear in analyzing the model. Even still, this project was quite challenging and the circuits we found surprised us. Interpreting real models will be even difficult. We encourage future work to apply interpretability tools cautiously and validate observational results with rigorous experimental tests.

\section*{Impact Statement}
This paper presents work whose goal is to make the function and mechanisms of deep neural networks interpretable to humans. We present methods for reasoning about counterfactual and out of distribution behavior in the models that we train. Though our setting is too small to be directly relevant to real-world use cases, we hope that similar techniques will be able to test, audit, and monitor deep neural networks that have been deployed in the real world. We also present results that urge caution and humility when attempting to interpret neural networks. We believe that robust and effective interpretability techniques may mitigate some societal harms that could arise from the use of deep neural networks, but that mistakenly trusting illusory interpretability techniques could be disastrous. 

\section*{Acknowledgements}

We would like to thank Coreweave for donating the computing resources that we used to run all of our experiments, to Bilal Chughtai for helpful discussions we had throughout the project, and to Neel Nanda for telling us to ``not hold back for fear of offending [him].'' We would also like to thank Nora Belrose, Neils uit de Bos, Aidan Ewart, Sara Price, Hailey Schoelkopf, Cédric Simal, and Benjamin Wright for their helpful feedback on earlier drafts of this paper.


\bibliography{camera_ready.bib}
\bibliographystyle{icml2024}

%%%%%%%%%%%%%%%%%%%%%%%%%%%%%%%%%%%%%%%%%%%%%%%%%%%%%%%%%%%%%%%%%%%%%%%%%%%%%%%
%%%%%%%%%%%%%%%%%%%%%%%%%%%%%%%%%%%%%%%%%%%%%%%%%%%%%%%%%%%%%%%%%%%%%%%%%%%%%%%
% APPENDIX
%%%%%%%%%%%%%%%%%%%%%%%%%%%%%%%%%%%%%%%%%%%%%%%%%%%%%%%%%%%%%%%%%%%%%%%%%%%%%%%
%%%%%%%%%%%%%%%%%%%%%%%%%%%%%%%%%%%%%%%%%%%%%%%%%%%%%%%%%%%%%%%%%%%%%%%%%%%%%%%
\newpage
\appendix
\onecolumn

\section{Author Contributions}

\paragraph{Dashiell}
Wrote the code for training and for calculating the Group Fourier Transform over $S_n$. Performed the initial analyses of models trained on $S_5$ and initially found what we came to call the coset circuit. Designed and ran causal experiments to confirm our understanding of the coset circuit. Derived formal properties of the coset circuit. Participated in discussions throughout the project and in writing the paper.

\paragraph{Qinan}
Ran training jobs and performed the bulk of circuit analysis on $S_6$, designed and ran ablation experiments and causal interchange interventions, participated in discussions throughout the project and the  writing of the paper.

\paragraph{Honglu}
Derived formal properties of the coset circuit, participated in the discussions throughout the project, and the writing of the paper.

%and wrote most of Appendix \ref{sec:group-theory} and parts of Appendix \ref{appendix:group-fourier-transform}, \ref{appendix:evidence} and \ref{appendix:irrep}.

\paragraph{Stella} Helped scope the problem and identify and plan the core experiments. Advised on the interpretation of the analysis and the writing of the paper.

\section{Structure of the Appendix}

In the Appendix we provide more of the mathematical background needed to fully describe some of our results and techniques. In particular, we explain the Group Fourier Transform and how we used to to analyze our models. We do this because we believe it is of independent interest and also because it is necessary to fully explain where our results and those of \citet{chughtai_toy_2023} diverge.

In Appendix \ref{appendix:details} we go over the precise experimental set up of the models that we trained.

In Appendix \ref{appendix:group-theory} we introduce the necessary concepts from group theory needed to rigorously talk about the more mathematical aspects of our results.

In Appendix \ref{appendix:rep-theory} we introduce representation theory, representations of the symmetric group, and the group Fourier transform.

In Appendix \ref{appendix:coset-circuit} we return to the coset circuit and coset neurons, with the presentation grounded in the mathematical concepts introduced in Appendices \ref{appendix:group-theory} and \ref{appendix:rep-theory}.

%In Appendix \ref{sec:toy-model-universality} we provide a detailed account of the GCR algorithm proposed by \citet{chughtai_toy_2023} and address directly how our results affect their conclusions.

Finally, in Appendix \ref{appendix:graphs} we present extra graphs that did not fit in the main paper and in Appendix \ref{appendix:irrep} we present a table of all of the conjugacy classes of subgroups of $S_5$.

\section{Experiment Details}
\label{appendix:details}
We conducted experiments focusing on the permutation group of $S_5$ and $S_6$. All models were trained on NVIDIA GeForce RTX 2080 GPUs. All models were implemented in PyTorch \citet{paszke_pytorch_2019} and trained with the Adam optimizer \citep{Kingma2014AdamAM} with a fixed learning rate of $0.001$, weight decay set to $1.0$, $\beta_{1} = 0.9$ and $\beta_{2} = 0.98$. At the beginning of each training run, the training set is sampled uniformly from all $|S_n|^2$ combinations of permutations. Each optimization step was made on the entire training set. Using our setup a single $S_5$ model trained in approximately 8 hours and a single $S_6$ model trained in approximately 100 hours, though multiple training jobs could be scheduled on a single GPU. Analysis and reverse engineering was performed with \citet{vink_pola-rspolars_2023, nanda2022transformerlens, harris2020array, GAP4, sage}.

\begin{table*}[htp]
\centering
\caption{Experiment hyperparameters.}\label{tab:exp_info}
\begin{tabular}{lccccc}
\toprule
Group  & \% Train Set & Num. Runs & Num. Epochs & Linear Layer Size & Embedding Size  \\
\midrule
$S_5$  & 40\% & 128 &  250,000 & 128 & 256 \\
%$S_5$ & medium & 50\% & 128  & 250,000 & 1024 & 512  \\
%$S_5$ & large & 50\% & 128  & 250,000 & 4096 & 512  \\
$S_6$  & 40\% & 100 & 50,000 & 256 & 512  \\
\end{tabular}

\end{table*}

\section{Group Theory}
\label{appendix:group-theory}

In this section, let us recall some basic definitions and propositions in group theory that are relevant to this paper.

\subsection{Groups}
A group $G$ is a nonempty set equipped with a special element $e\in G$ called the \emph{identity} and a multiplication operator $\cdot$ satisfying the following:
\begin{itemize}
    \item (inverse) For each element $a\in G$, there exists an element $b\in G$ such that $a\cdot b = b\cdot a = e$.
    \item (identity) For each element $a\in G$, $a\cdot e = e\cdot a = a$.
    \item (associativity) For elements $a, b, c\in G$, we have $(a\cdot b)\cdot c = a\cdot (b\cdot c)$.
\end{itemize}

The inverse of $a\in G$ is denoted by $a^{-1}$.

\begin{example}
    The set of integers $\mathbb Z$ along with the addition $+$ form a group. The identity element is $0$. Also with the addition, the same is true for the set of rational numbers $\mathbb Q$, the set of real numbers $\mathbb R$ and the set of complex numbers $\mathbb C$.
\end{example}
\begin{example}
    The symmetric group introduced in Section \ref{sec:permutations-basics} along with the composition of permutations satisfies the group axioms. The identity element is the identity permutation leaving each element unchanged.
\end{example}
\begin{example}
    The set of natural numbers $\mathbb N$ and addition \emph{do not} form a group. The reason being that the inverse elements do not exist except for $0$.
\end{example}

\begin{definition}
    Given a group $G$, a subgroup $H$ is a subset of $G$ such that
    \begin{itemize}
        \item $a\cdot b\in H$ for any $a, b\in H$.
        \item $e\in H$.
        \item $a^{-1}\in H$.
    \end{itemize}
\end{definition}

One can check that $H$ along with the multiplication satisfies the group axiom as well.
$H$ being a subgroup of $G$ is denoted by $H \le G$.

\subsection{Cosets and double cosets}
\label{appendix:double-cosets}

\begin{definition}
Given a proper subgroup $H<G$ and an element $g\in G$, the set $gH:=\{gh~|~h\in H\}$ is called a left $H$-coset. Similarly, $Hg:=\{hg~|~h\in H\}$ is called a right $H$-coset.

$gH$ is sometimes called a coset if the subgroup $H$ is clear from the context. When we do not mention whether it is a left coset or a right coset, left coset is the default.
\end{definition}

\begin{lemma}
Two cosets $g_1H$ and $g_2H$ are either the same subset of $G$ or disjoint (i.e., $g_1H \bigcap g_2H = \emptyset$).
\end{lemma}

\begin{lemma}
If $G$ is a finite group, any two $H$-cosets have the same number of elements.
\end{lemma}

As a result, one can pick suitable representative elements (but not unique) $g_1, \cdots, g_n\in G$, so that $g_1H, \cdots, g_nH$ form a partition of $G$. Because the cosets have equal sizes, we can also conclude that $|G|$ is always divisible by $|H|$.

\begin{definition}
Given two subgroups $H, L<G$ and an element $g\in G$, the set $HgL:=\{hgl~|~h\in H, l\in L\}$ is called the $(H, L)$-double coset, or the double coset if the pair $(H, L)$ is clear from the context.
\end{definition}

Double cosets enjoy the similar property as cosets:
\begin{lemma}
Two double cosets $Hg_1L$ and $Hg_2L$ are either the same or disjoint.
\end{lemma}

As a result, $G$ can be similarly decomposed as a disjoint union of $(H, L)$-double cosets.
However, when $G$ is finite, $(H, L)$-double cosets do not always come with equal sizes. So the decomposition is not equal-sized.

For simplicity, we call the $(H, H)$-double coset the $H$-double coset.

\subsection{Normal Subgroups}
\begin{definition}
A subgroup $N$ is \emph{normal} in $G$, denoted $ N \trianglelefteq G$, if for any $g \in G$ and any $n \in N$, we have $gng^{-1} \in N$. 
\end{definition}

A subgroup is normal if and only if the left and right cosets are the same, i.e., for any $g \in G,\ gN = Ng$. Normal subgroups are important because they are precisely the groups for which the set of $N$-cosets $G/N$ has a natural group structure.

\begin{definition}
Given a group $G$ and a normal subgroup $N \trianglelefteq G$, the \emph{quotient group} $G/N$ is defined to be the set of $N$-cosets endowed with the multiplication given by $gN\cdot hN=ghN$ for any $g, h\in G$.
\end{definition}

The well-definedness of the multiplication is a consequence of $N$ being normal and its group axioms are straightforward to check.

%A motivating example can be seen in modular addition. Consider $\mathbb{Z}/12\mathbb{Z}$, the group of addition $\mod 12$. The elements that are divisible by $3$ form a subgroup of  $\mathbb{Z}/12\mathbb{Z}$ because if $x$ and $y$ are divisible by $3$ then so is $x + y$, and they form a normal subgroup because $\mathbb{Z}/12\mathbb{Z}$  is abelian. The cosets of this subgroup are precisely the congruence classes of integers $mod 3$ and we can see how the property of coset multiplication gets instantiated.

%In this way multiplication in the group get split into two easier problems, multiplication in the normal subgroup $N$ and the quotient group $G \ N$. Returning to the above example, if we wish to compute $5 + 9 \in \mathbb{Z}/12\mathbb{Z}$ we can solve it by breaking it down into two sub-problems: addition in $\mathbb{Z}/4\mathbb{Z}$, the subgroup, and  addition in $\mathbb{Z}/3\mathbb{Z}$ the quotient group. In $\mathbb{Z}/4\mathbb{Z}$ the addition is $(1 + 1) \equiv 2\mod 4$ and in $\mathbb{Z}/3\mathbb{Z}$ it is $(2 + 0) \equiv 2 \mod 3$. The only number less than $12$ that is congruent to $2 \mod 3$ and $2 \mod 4$ is $2$, and that is the answer.
\begin{example}
If $G$ is commutative (for every $g, h\in G$, we have $gh=hg$), every subgroup $H\le G$ is normal.
\end{example}

\begin{example}
If $G=S_n$, the subgroup $S_{n-1}$ fixing the first element is \emph{not} a normal subgroup. On the other hand, the alternating subgroup $A_n$ (consisting of even permutations) is a normal subgroup of $S_n$.
\end{example}

The double cosets of a normal subgroup are simply the usual cosets.

\begin{lemma}
Given a normal subgroup $H \trianglelefteq G$, the left $H$-coset and the right $H$-coset are in one-to-one correspondence. Furthermore, the set of $H$-double cosets is also in one-to-one correspondence to $H$-cosets.
\end{lemma}
\begin{proof}
By definition, $gHg^{-1} = H$. Therefore, $gH = Hg$. $HgH = gHH = gH$.
\end{proof}

\subsection{Conjugate Subgroups}
The cosets of a \emph{normal} subgroup $N \trianglelefteq G$ themselves form a group. If $x,\;y \in G$ and  $x \in gN$ but $\;y \in hN$, then   $xy \in ghN$. If $G$ is not abelian, however, many or even all subgroups are not normal and do not have this property. For a non-normal subgroup $H$, a $g \notin H$ gives rise to a different \emph{conjugate} subgroup $gHg^{-1}$.

In general, the relationship between the cosets of $H$ and $gHg^{-1}$ is complex, but they will have at least one left and one right coset in common: $Hg^{-1} = g^{-1}(gHg^{-1})$. Every right coset  $Hx$ will have a left coset pair $y(gHg^{-1})$ such that when multiplied, right coset on the left and left coset on the right, $Hxy(gHg^{-1}) = Hg^{-1}$, specifically when $xy = g^{-1}$.

This relationship between the cosets of pairs of conjugate subgroups is not as powerful as that of the cosets of normal subgroups, but conjugate subgroups are guaranteed to exist in non-abelian groups, whereas there are many simple groups without normal subgroups at all. 

This relationship between pairs of conjugate subgroups is also useful enough that it is used by every model we trained. 
In general, we have the following:
\begin{lemma}
For any $H\le G$ and an element $g\in G$, the set of conjugate elements $gHg^{-1}$ forms a subgroup of $G$.
\end{lemma}
If the conjugate subgroup $gHg^{-1}$ is different than $H$, the left and right cosets $gH, Hg$ are different. 

The double coset circuits operate by first identifying a pair of different conjugate subgroups $H$ and $gHg^{-1}$. It exploits the fact that the left coset $gH$ and the right coset $(gHg^{-1})g$ are the same subset of $G$, which will be fully generalized and elaborated in the later sections.

\subsection{An important case}
When a group $G$ decomposes as only two disjoint $H$-double cosets, any pair of subgroups conjugate to $H$ shares a left coset with another's right coset.
\begin{lemma}
    Let $H_1, ..., H_n$ be conjugate subgroups of $G$, such that for each $H_i$ the double coset $H_{i}gH_{i}$ is equal to either $H_i$ or $G \setminus H_i$. Then for each pair of subgroups $H_i$ and $H_j$ there exists a $g \in G$ such that $H_{i}g = gH_{j}$. Moreover, the only double cosets of $H_{i}$ and $H_j$ are $H_{i}gH_j = gH_j$ and $H_{i}xH_{j} = G \setminus gH_j$.
\end{lemma}
\begin{proof}
    If $i=j$, for any $h \in H_i$ the shared coset is the subgroup itself.
    If $i \neq j$, because $H_i$ and $H_j$ are conjugate, there exists a $g \in G$ such that $H_j = g^{-1}H_{i}g$. The left coset is equal to the right coset:
    $$gH_j = g(g^{-1}H_{i}g) = H_{i}g$$
    Notice that the double coset $H_{i}gH_j = H_{i}(H_{i}g) = H_{i}g$. But for $x \neq g$:
    \begin{align}
        H_{i}xH_j &= H_{i}xg^{-1}H_{i}g \\
        &= (G \setminus H_{i})g \\
        &= G \setminus H_{i}g
    \end{align}
\end{proof}

\section{Representation Theory}\label{appendix:rep-theory}
\subsection{Preliminaries}

\begin{definition}
Given a group $G$, a \emph{representation of $G$} is a group homomorphism $\rho_V: G\rightarrow GL(V)$ for some finite (but nonzero) dimensional vector space $V$ over a field $k$. When we do not specifically mention $k$, we use $\mathbb C$ as the default.
\end{definition}
In other words, a representation maps a group element $g$ to a linear operator $f(g): V\rightarrow V$ where $V$ is a vector space of dimension $d$, so that the group multiplication becomes compositions of linear operators ($f(g\cdot h) = f(g)\circ f(h)$). Without explicit specifications, all representations in this paper are assumed to be over complex numbers. Recall also that finite dimensional linear operators can be represented as matrices, and composition of linear operators is then given as matrix multiplication.

When the context is clear, sometimes we omit the subscript $V$ in the notation $\rho_V$.

The representations of finite groups have a rich and beautiful theory (see \citet{diaconis_persi_group_1988, fulton_representation_1991}). Here, we recall a few basic definitions and facts without going into details.
\begin{definition}
A representation $\rho_V: G\rightarrow GL(V)$ is a sub-representation of $\rho_W: G\rightarrow GL(W)$ if $V$ can be identified as a linear subspace of $W$ so that $\rho_W(g)$ restricts to $\rho_V(g)$ for all $g\in G$.
\end{definition}

\begin{example}
For any group $G$, the map $G\rightarrow GL(V)$ sending all elements to the identity matrix is a representation. When $\text{dim}(V) = 1$, we call it the \emph{trivial representation} of $G$.
\end{example}

\begin{definition}
Given two representations $\rho_V, \rho_W$ of $G$, the direct sum of vector spaces $V\oplus W$ admits a natural representation of $G$ by letting $\rho_V, \rho_W$ act on each component separately. We call this the direct sum of representations $\rho_V, \rho_W$, and denote it by $\rho_V \oplus \rho_W$.
\end{definition}

\begin{definition}
Similarly, given two representations $\rho_V, \rho_W$, the tensor product $V\otimes W$ admits a natural representation of $G$ by acting on $V, W$ separately and extend by linearity. We call this the tensor product of representations $\rho_V, \rho_W$, and denote it by $\rho_V \otimes \rho_W$.
\end{definition}

\begin{definition}
A representation $\rho$ of a group $G$ is \emph{irreducible}, if it does not have sub-representations other than $\rho$.
\end{definition}

We denote the set of all irreducible representations of $G$ by $\irr(G)$

\begin{lemma}
A representation $\rho$ of a finite group $G$ is a direct sum of irreducible representations.
\end{lemma}

\begin{example}
The trivial representation of $G$ is irreducible.
\end{example}

\begin{example}
    The permutation representation maps $S_n \rightarrow GL(\mathbb{C}^3)$, i.e. $3\!\times\!3$ matrices with a single $1$ in each row and column and zeros everywhere else.
    \[
        (2 \; 1 \; 3 ) \mapsto \begin{pmatrix}
        0 & 1 & 0 \\
        1 & 0 & 0 \\
        0 & 0 & 1 \\
    \end{pmatrix} \quad
        (3 \; 2 \; 1) \mapsto\begin{pmatrix}
        0 & 0 & 1 \\
        0 & 1 & 0 \\
        1 & 0 & 0 \\
    \end{pmatrix}
    \]

    You can see that the matrices of the permutation representation act on the basis vectors of $\mathbb{C}^3$:
    \[
    \begin{pmatrix}
        0 & 0 & 1 \\
        0 & 1 & 0 \\
        1 & 0 & 0 \\
    \end{pmatrix}
    \begin{pmatrix}
        x \\ y \\ z
    \end{pmatrix} = 
    \begin{pmatrix}
        z \\ y \\ x
    \end{pmatrix}
    \]

    What it means to be a representation is that the group multiplication becomes matrix multiplication, so just as $(2 \; 1 \; 3 )(3 \; 2 \; 1) = (2 \; 3 \; 1 )$,
    \[
    \begin{pmatrix}
        0 & 1 & 0 \\
        1 & 0 & 0 \\
        0 & 0 & 1 \\
    \end{pmatrix}
    \begin{pmatrix}
        0 & 0 & 1 \\
        0 & 1 & 0 \\
        1 & 0 & 0 \\
    \end{pmatrix} = \begin{pmatrix}
        0 & 0 & 1\\
        1 & 0 & 0 \\
        0 & 1 & 0
    \end{pmatrix}
    \]
\end{example}

\begin{example}
    The permutation representation is \emph{reducible}, because there is a subspace of $\mathbb{C}^3$ that is invariant to it's action.
    \[
    \begin{pmatrix}
        0 & 0 & 1 \\
        0 & 1 & 0 \\
        1 & 0 & 0 \\
    \end{pmatrix}
    \begin{pmatrix}
        x \\ x \\ x
    \end{pmatrix} = 
    \begin{pmatrix}
        x \\ x \\ x
    \end{pmatrix}
    \]

    Note that there is no permutation matrix acting on the vector $\begin{pmatrix}
        x & x & x
    \end{pmatrix}^T$ that will change it, because all of the components are equal.

    As it turns out, there are no \emph{irreducible} representations of $S_3$ that are three-dimensional. The largest irrep of $S_3$ is $\rho_{(2,1)}$, which is made of $2\!\times\!2$ matrices. The matrices of the $(2,1)$ irrep of $S_3$ are as follows:

    \begin{flalign*}
        (1 \; 2 \; 3) &\mapsto \begin{pmatrix}
            1 & 0 \\
            0 & 1 
        \end{pmatrix} &
        (2 \; 1 \; 3 ) &\mapsto \begin{pmatrix}
            -1 & 0 \\
            0 & 1 
        \end{pmatrix} &
        (3 \; 2 \; 1) &\mapsto   
        \begin{pmatrix}
            1/2 & -\sqrt{3}/2 \\
            \sqrt{3}/2 & -1/2
        \end{pmatrix} \\
        (1 \; 3 \; 2) &\mapsto   
        \begin{pmatrix}
            -1/2 & \sqrt{3}/2 \\
            \sqrt{3}/2 & 1/2
        \end{pmatrix} & 
        (3 \; 1 \; 2) &\mapsto \begin{pmatrix}
            -1/2 & \sqrt{3}/2 \\
            -\sqrt{3}/2 & -1/2
        \end{pmatrix} &
        (2 \; 3 \; 1) &\mapsto \begin{pmatrix}
            -1/2 & \sqrt{3}/2 \\
            -\sqrt{3}/2 & -1/2
        \end{pmatrix}
    \end{flalign*}
    
    We leave it as an exercise to the reader to verify that $\rho_{(2,1)}(2 \; 1 \; 3 )\rho_{(2,1)}(3 \; 2 \; 1) = \rho_{(2,1)}(2 \; 3 \; 1)$.

\end{example}

Trace is an important notion in linear algebra. Taking trace of a representation induces an important map from $G$ to $\mathbb C$.

\begin{definition}
Let $\rho_V$ be a representation of $G$. The \emph{character} of $\rho_V$ is a map $\chi(\rho_V):G\rightarrow \mathbb C$ given by $\chi(\rho_V)(g) = \trace(\rho_V(g))$.
\end{definition}

\begin{lemma}
The character $\chi(\rho_V)$ takes the same value on a conjugacy class of $G$. In other words, $\chi(\rho_V)(h) = \chi(\rho_V)(ghg^{-1})$.
\end{lemma}

To distill this property for a wider range of functions, we have the following definition:
\begin{definition}
Let $f:G\rightarrow \mathbb C$ be a map. If $f(h)=f(ghg^{-1})$ for any $g, h\in G$, $f$ is called a \emph{class function}.
\end{definition}

For a finite group $G$, the set of class functions form a finite-dimensional vector space. There is an important inner product between class functions.
\begin{definition}
The inner product of two class functions $\phi, \psi$ are defined as:
\[
\langle \phi, \psi\rangle = \dfrac{1}{|G|}\sum\limits_{g\in G} \phi(g)\overline{\psi(g)}.
\]
\end{definition}

As we require the class functions to take the same values on conjugacy classes, the dimension of the vector space of class functions is equal to the number of conjugacy classes in $G$. On the other hand, we have the following important theorem:
\begin{theorem}
The characters of $\irr(G)$ forms an orthonormal basis in the vector space of class functions.
\end{theorem}

\begin{lemma}
For a finite group $G$, $\irr(G)$ is a finite set. Furthermore, the order of $\irr(G)$ is equal to the number of conjugacy classes in $G$.
\end{lemma}

\section{Fourier transform over finite groups}
\label{appendix:group-fourier-transform}

%We use the \emph{group} Fourier transform over $S_n$ extensively in our analysis. 

Despite being mostly perceived as a powerful tool in physics and engineering, the Fourier transform has also been successfully applied in group theory thanks to its generalization to locally compact abelian groups as well as an analog over finite groups. The purpose the group Fourier transform serves is largely analogous to the one served by the classical Fourier transform: it provides an alternate orthogonal basis with which to analyze functions from a group $G$ to either $\mathbb{R}$ or $\mathbb{C}$.

To motivate the transition from the classical Fourier theory to the Fourier theory over groups, we start with a brief recall of the definitions.

The classical Fourier transform over real numbers converts a complex-valued Lebesgue-integrable function $f: \mathbb R\rightarrow \mathbb C$ into a function from the complex unit circle $S^1$ to $\mathbb{C}$ with following formula:
\begin{equation}\label{eqn:classical-ft}
\hat f(\xi) = \displaystyle\int_{-\infty}^{\infty} f(x)e^{-2\pi i\xi x}dx.
\end{equation}

Taking one step further in abstraction, we note that $e^{-2\pi i\xi x}$ as a function of $x$ has the defining properties of turning additions into multiplications (being a group homomorphism) and always having complex norm $1$:
\begin{align*}
    e^{-2\pi i\xi (x_1 + x_2)} &= e^{-2\pi i\xi x_1} \cdot e^{-2\pi i\xi x_2}, \\
|e^{-2\pi i\xi x}| &= 1.\\
\end{align*}

We call such functions the \emph{characters} of $\mathbb R$, though they are often thought of as \emph{frequencies}. One can prove that all characters of $\mathbb R$ can be written as $e^{-2\pi i\xi x}$ for a suitable $\xi\in\mathbb R$.

Looking back at \eqref{eqn:classical-ft}, the properties we need in order to define the Fourier transform over $\mathbb R$ are:
\begin{itemize}
    \item $\mathbb R$ has the Lebesgue measure (allowing for integration to happen).
    \item $\mathbb R$ is a group (so that the characters make sense as group homomorphisms from $\mathbb R$ to the unit circle group $S^1\subset\mathbb C$).
\end{itemize}

Now, if we are given a finite group $G$, the Fourier transform of a finite group is an operator converting a map $f:G\rightarrow \mathbb C$ into a function between $\irr(G)$ and the set of linear operators $M(V)$.

\begin{definition}
Given a group $G$, the \emph{Fourier transform} of a map $f:G\rightarrow \mathbb C$ is a function $\hat{f}$ from $\irr(G)$ to the union of $M(\mathbb C^n)$ for all $n$ such that 
\[
\hat{f}(\rho) = \sum\limits_{a\in G} f(a)\rho(a)
\]
for an irreducible representation $\rho$.
\label{ft}
\end{definition}

The analogy comes from the following similar facts:
\begin{itemize}
    \item $G$, as a finite set, has the invariant discrete measure (where the ``integration" becomes the sum).
    \item $G$ is a group, and the irreps $\rho$ are in a sense the "smallest" group homomorphisms from $G$ to $GL(n, \mathbb C)$ (note that the images of $\rho$ similarly have complex-norm-$1$ determinants due to $G$ being a finite group).
\end{itemize}

For more details and applications, one can refer to, for example, \citet{steinfourier}. We would like to note that there is also an inverse transform that restores the original function $f$ from $\hat{f}$:
\begin{equation} 
     f(g) = \frac{1}{|G|} \sum_{\rho \in \irr(G)} d_{\rho} \trace [ \hat{f}(\rho) \rho(g^{-1}) ]
     \label{ift}
\end{equation}

\section{The Coset Circuit (with more math)}
\label{appendix:coset-circuit}

We did not introduce it in the main body of our paper because it would distract from the core of our results, but for the first half of our investigation the Fourier transform over the symmetric group was integral to our investigation. We were building directly on \citep{chughtai_toy_2023} who had shown striking results around the weights of single-layer models showing high degrees of correlation with the irreps of the symmetric group. We wished to cast those results in the language of the group Fourier transform. Even when we realized that the mechanism of the model was based around cosets it became extremely important to understand why our coset circuit was so concentrated in Fourier space.

\subsection{Harmonic Analysis on the Symmetric Group}
\label{sec:fourier-analysis}

The presentation in the Appendix \ref{appendix:rep-theory} was given in terms of functions on $\mathbb{C}$ because it is required for arbitrary groups. For $S_n$ all of the irreps are rational \citep{fulton_representation_1991} and the Fourier transform of functions on $S_n$ can safely be defined over $\mathbb{R}$.

In this section we describe how we use the Fourier transform to analyze the weights and activations of an MLP. The inputs to the model are two one-hot vectors, $\mathbf{x}_{l},\;\mathbf{x}_{r}$, which multiply the embedding matrices $\mathbf{E}_{l}\mathbf{x}_{l}$ and $\mathbf{E}_{r}\mathbf{x}_{r}$. $\mathbf{E}_{l}$ and $\mathbf{E}_{r}$ are $d \times |G|$ matrices, where $d$ is the embedding dimension and $|G|$ is the size of the group. The \emph{columns} are the embedding vectors for a single element $g \in G$. The normal approach would be to try and look at the column spaces of $\mathbf{E}_l$ and $\mathbf{E}_r$, as these columns are the inputs to the model. However, since each \emph{row} of $\mathbf{E}_l$ and $\mathbf{E}_r$ and each value of that row is associated with a single element of $G$, we instead treat each \emph{row} of the embedding as a function $f: G \rightarrow \mathbb{R}$.

In fact, anywhere in the model where a matrix or set of activations has $|G|$ in the shape we can expand into the Fourier basis.
For non-abelian groups, each Fourier frequency is an irrep, and the Fourier transform for each irrep is matrix-valued. This is, on its face, \emph{less} interpretable than what we started with. Following the techniques outlined in \citet{diaconis_persi_group_1988},  however, we can expand the function at each element $g \in G$ into a new Fourier basis. Concretely, if our function $f: G \rightarrow \mathbb{R}$ is represented as a vector, we know from \ref{ift} that each element of the vector is a sum of the Fourier components:

\[
\begin{bmatrix}
    f(g_{1}) \\
    f(g_{2}) \\
    \vdots \\
    f(g_{|G|})
\end{bmatrix} = 
\frac{1}{|G|}
\begin{bmatrix}
    \sum_{\rho} d_{\rho}\trace[\hat{f}(\rho)\rho(g^{-1}_{1})] \\
    \sum_{\rho} d_{\rho}\trace[\hat{f}(\rho)\rho(g^{-1}_{2})] \\
    \vdots \\
    \sum_{\rho} d_{\rho}\trace[\hat{f}(\rho)\rho(g^{-1}_{|G|})]
\end{bmatrix}
\]

We can keep track of all of the Fourier components at once by purposefully not completing the sum from \ref{ift}), but instead keep each term into a new dimension:

\[
\frac{1}{|G|}
\begin{bmatrix}
    d_{\rho_{1}}\trace[\hat{f}(\rho_{1})\rho_{1}(g^{-1}_{1})] & \dots & d_{\rho_{k}}\trace[\hat{f}(\rho_{k})\rho_{k}(g^{-1}_{1})] \\
    d_{\rho_{1}}\trace[\hat{f}(\rho_{1})\rho_{1}(g^{-1}_{2})] & \dots & d_{\rho_{k}}\trace[\hat{f}(\rho_{k})\rho_{k}(g^{-1}_{2})] \\
    \vdots & & \vdots \\
    d_{\rho_{1}}\trace[\hat{f}(\rho_{1})\rho_{1}(g^{-1}_{|G|})] & \dots & d_{\rho_{k}}\trace[\hat{f}(\rho_{k})\rho_{k}(g^{-1}_{|G|})] \\
\end{bmatrix} 
\]

Though this may seem like it is only making the data more complicated, it gives us many tools for analyzing the data. In particular, it turns out that the weights and activations are sparse in this new basis, which gives us a small path forward in analyzing the mechanisms.

\begin{corollary}
If $H_i$ and $H_j$ are conjugate subgroups of $G$ such that the only two double cosets are $H_{i}gH_{j}$ and $H_{i}H_{j}$, then each right coset $H_{i}x$ has a paired left coset $yH_j$ where $y = x^{-1}g$ such that for all $h_{x} \in H_{i}x$ and $h_y \in yH_j$,  $h_{x}h_y \in H_{i}gH_{j}$
\end{corollary}

\begin{lemma}
Let $f: G \rightarrow \mathbb{C}$ be constant on the cosets of $H \le G$ and non-zero on at least one coset. Then $\hat{f}(\rho) = 0$ if the restriction of $\rho$ to $H$, $\rho |_{H}$ does not contain the trivial representation as a subrepresentation.
\label{constant_on_cosets}
\end{lemma}

\begin{proof}

The function $f$ can be decomposed as the sum of functions 
\[ f_{xH}(\sigma) = \begin{cases} \alpha_{x} \quad \sigma \in xH \\ 0 \quad \text{otherwise} \end{cases} \]
for each coset $xH$. Because Fourier transform $\hat{f}$ is invariant under translation we may, without loss of generality, analyze only the function $f_{H}$. For a given $\alpha_{x}$, $\hat{f}_{xH}(\rho) = \hat{f}^{x}_{H}(\rho) = \rho(x)\hat{f}_{H}(\rho)$ for all $x \in G$. Recall the definition of $\hat{f}_{H}(\rho)$ from \ref{ft}:
\begin{align}
\hat{f}_{H}(\rho) &= \sum_{g \in G} f_{H}(g)\rho(g) \\
 &= \alpha_{H} \sum_{h \in H} \rho|_{H}(h) \\
 \label{decomp}
 &= \alpha_{H} \sum_{h \in H} T^{-1} [\bigoplus_{\tau_{i} \in \mathcal{T}}\tau_{i}(h)]T \\ 
 &= \alpha_{H} T^{-1} [\bigoplus_{\tau_{i} \in \mathcal{T}}  \sum_{h \in H} \tau_{i}(h)] T
\end{align}
where in \ref{decomp} we decompose $\rho|_{H}$ into a direct sum of irreps of $H$. But because each $\tau_{i}$ is irreducible, $\sum_{h \in H} \tau_{i}(h) = \mathbf{0}$ unless $\tau_i$ is the trivial irrep. Thus, unless the decomposition of $\rho_H$ into irreps of $H$ includes the trivial representation, $\hat{f}|_{H} = 0$
\end{proof}

\subsection{Logits and Counting Cosets}\label{counting_cosets}

In \citet{chughtai_toy_2023}, one way of justifying the GCR algorithm is to study the correlation between the character functions and the neuron activations. We would like to argue that the correlation between the GCR and the coset membership counting function may already exist, and in some simple cases it can be made explicit. 

More precisely, we are measuring the correlation between the character function $\chi(\rho)$ of an irrep $\rho$ with a set function $f:G\rightarrow \mathbb C$. In this section, we provide an explicit characterization of $f$ in terms of trace and irreps, when $f$ counts the membership of cosets.

We are specifically interested in the following situation:
\begin{lemma}
    Suppose $f: G\rightarrow \mathbb C$ is a function such that its Fourier transform $\hat f$ is nonzero only on an irreducible representation $\rho$ and the trivial representation. Let $\hat f(\rho) = A \in M(\mathbb C^n)$. We have the following explicit formula:
    \begin{equation}\label{eqn:fourier_inv}
    f(\sigma) = \dfrac{d_\rho}{|G|}\trace(A\cdot \rho(\sigma^{-1})) + \dfrac{|H|}{|G|}.
    \end{equation}
\end{lemma}
\begin{proof}
    This is immediate by the Fourier inversion formula.
\end{proof}

In this case, although $f$ is not directly written in terms of $\trace(\rho(\sigma^{-1}))$, $f$ is correlated with $\trace(\rho(\sigma^{-1}))$ depending on how much $A$ is concentrated to the diagonal and how even are the diagonal entries. For the rest of the section, we show that under certain conditions, Equation \eqref{eqn:fourier_inv} applies verbatim to the functions that count membership of cosets for a collection of conjugate subgroups.

Given a subgroup $H\le G$, let $1_H$ be the function that takes value $1$ on the subgroup $H$, and takes $0$ otherwise. The action of $G$ on cosets $G/H$ induces a representation of $G$ on $\mathbb C^{|G/H|}$ by permuting the basis accordingly. We call it the \emph{permutation representation} of $G$ on $G/H$.
\begin{lemma}
     The Fourier transform of $1_H$ is nonzero only at the irreducible components of the permutation representation of $G$ on $G/H$.
\end{lemma}
\begin{proof}
    By definition, the Fourier transform of $1_H$ on an irrep $\rho$ is
    \[
    \widehat{1_H}(\rho) = \sum\limits_{a\in H} \rho(a).
    \]
    Notice that the image of $\sum\limits_{a\in H} \rho(a)$ are invariant under $H$ due to the symmetry of this expression.
    
    Let $V$ be the vector space where $\rho$ acts on. Under the action of the subgroup $H$ through $\rho$, one can decompose $V$ as irreps of $H$. We group them into two parts:
    \[
    V = V^H \oplus V',
    \]
    where $V^H$ is a direct sum of copies of trivial representation of $H$ (or in other words, the invariant subspace of $V$ under $H$), and $V'$ is the direct sum of nontrivial irreducible components of $V$.

    We immediately see the following by definition:
    \[
    \sum\limits_{a\in H} \rho(a)|_{V^H} = |H|\cdot \text{Id}_{V^H}.
    \]

    Also by definition, nontrivial irreps of $H$ do not have invariant subspaces since they do not admit proper sub-representations. Therefore, 
    \[
    \sum\limits_{a\in H} \rho(a)|_{V'} = 0.
    \]

    As a result, $\widehat{1_H}(\rho)$ is simply a scaled projection to the invariant subspace of $V$. Whether it is zero depends on whether $\text{Res}_H \rho$ has any trivial components.
    
    By Frobenius reciprocity,
    \[
    \langle \text{Ind}^G_H(1_H), \chi(\rho) \rangle = \langle 1_H, \chi(\text{Res}_H(\rho)) \rangle_H,
    \]
    where $\chi(\rho)$ is the character of the irrep $\chi\in\irr(G)$ given by its traces, and $\langle\cdot\rangle$ is the inner product between class functions.

    The left-hand side $\langle \text{Ind}^G_H(1_H), \chi(\rho) \rangle$ is nonzero if and only if $\rho$ is an irreducible component of the permutation representation of $G$ on $G/H$. The right-hand side $\langle 1_H, \chi(\text{Res}_H(\rho)) \rangle_H$ is nonzero if and only if $\text{dim}(V^H) \neq 0$

\end{proof}

Note that this lemma also works for $1_{gH}$ for a coset $gH$, since Fourier transforms turns the translation action by $g$ into group multiplication by $\rho(g)$. 

In the double coset circuit, we are specifically interested in the membership counting functions. More specifically, let $H_1, \cdots, H_n$ be a collection of conjugate subgroups of $G$. Given an element $\sigma\in G$, define the membership counting function as
\[
F(\sigma) = \sum\limits_{i=1}^n 1_{\sigma H_i}.
\]

Combining all previous results, we have the following corollary describing the membership counting function $F$.
\begin{corollary}
    If the permutation representation of $G$ on $G/H_1$ has only $2$ irreducible components, the Fourier transform $\hat F$ of the membership counting function $F$ is nonzero only at these $2$ irreducible components. In particular, the equation \eqref{eqn:fourier_inv} applies to $F$.
\end{corollary}

One may wonder how restrictive it is for the permutation representation on $G/H$ to only have $2$ irreducible components. The follow lemma shows that it applies to our case when $G=S_n$ and $H=S_{n-1}$.
\begin{lemma}
    For $S_n$ and the subgroup $S_{n-1}$ fixing one element, the permutation representation has only two irreducible components.
\end{lemma}
\begin{proof}
    The natural representation of $S_n$ on $\mathbb C^n$ (by permuting the basis) decomposes as a direct sum of trivial representation and the standard representation of dimension $n-1$.
\end{proof}

Indeed, we see that when looking at the action of an individual neuron on the prediction space (i.e. ``if this neuron fires, which predictions become more likely and which less?''), we see that it is only neurons that are predicting the \textit{same coset} that are correlated. The average pairwise correlation of neuron actions is uncorrelated, as is the correlation of neurons associated with the same irrep. Refer to Table \ref{tab:unembed_corr} for the full results.

\begin{table}[h]
    \centering
     \caption{The correlation of unembedding neurons. Neurons that correspond to the same coset are averaged together in the unembedding, leading to the unembedding vectors being highly correlated.}
    \begin{tabular}{lcc}
    \toprule
        &  Mean Correlation   &   Std Dev Correlation   \\
    \midrule
       Within Coset &  0.814 & 0.445 \\ 
    \midrule
       Within Subgroup Conjugacy Class & -0.002 & 0.222 \\
    \midrule
       Baseline & -0.003 & 0.163 \\
    \bottomrule
    \end{tabular}
   .
    \label{tab:unembed_corr}
\end{table}

\subsection{An Asymptotic Analysis}

Our theory of coset circuits and the GCR algorithm of \citep{chughtai_toy_2023} cannot be equivalent because there is no one-to-one relationship between irreps and subgroups. Even for $S_5$, there are more subgroups than irreps. Quantitatively speaking, the irreps already fail to catch up with the number of subgroups. For the direct comparison of $S_n$ refer to 

Asymptotically, the number of subgroups of $S_n$ is bounded below as follows (see \citet[Corollary 3.3]{pyber}):
\[
2^{(\frac{1}{16} + o(1))n^2} \leq |\text{Sub}(S_n)|,
\]
whereas the number of irreps of $S_n$ is asymptotically the following (see \citet{erdos}):
\[
|\irr(S_n)| \sim \dfrac{1}{4n\cdot 3^{\frac{1}{2}}} e^{\pi(\frac{2}{3})^{\frac{1}{2}}n^{\frac{1}{2}}}.
\]
We see that the former has a much higher asymptotic growth than the latter.

In practice, as can be seen in Table \ref{tab:s5_subgroups}, many subgroups concentrate on more than one irrep. We do not have an explanation for why the coset circuits always do concentrate one irrep. In practice, the different values for the cosets are arranged so that the contributions of all but one irrep cancel out. We hypothesize that it may have something to do with the margin maximization effect discussed in \citep{morwani2024feature}. As we mention in the main body, we observe that subgroups which concentrate on more than one irrep will form coset circuits that concentrate entirely on any of the irreps, while still behaving equivalently. We do not think that there is in fact a connection between what the circuit is doing the irrep.

\begin{table}[h]
    \centering
    \caption{The number of subgroups and the number of irreps from $S_5$ to $S_{12}$. The numbers of subgroups use the A005432 sequence of the OEIS \citep{oeis}. The numbers of irreps corresponds to the number of integer partitions of $n$ and use the A000041 sequence of the OEIS \citep{oeis}.}
    \begin{tabular}{lccccccccc}
    \toprule
        &   $S_5$   &   $S_6$   &   $S_7$   &   $S_8$   &   $S_9$   &   $S_{10}$  &   $S_{11}$ & $S_{12}$  \\
    \midrule
       Number of subgroups &  $156$ & $1455$ & $11300$ & $151221$ & $1694723$ & $29594446$ & $404126228$ & $10594925360$  \\ 
    \midrule
       Number of irreps & $7$ & $11$ & $15$ & $22$ & $30$ & $42$ & $56$ & $77$ \\
    \bottomrule
    \end{tabular}
    
    \label{tab:sn_subgroups_sequence}
\end{table}

%\begin{lemma}
%    The same set-up as above, we have ... (something about the trace?? combining the above 2)
%\end{lemma}
%\begin{proof}
%    ......
%    ......
%    $1_{H}(\sigma)$ is either equal to 0 or 1, so solving for the trace, when $\sigma \in H$, then $\trace [\hat{f}(\rho)\rho(\sigma^{-1})] = \frac{1}{d_{\rho}}(|G| - |H|)$ and when $\sigma \notin H$ then $\trace[\hat{f}(\rho)\rho(\sigma^{-1})] = -\frac{1}{d_{\rho}}|H|$ 
%\end{proof}

%------original-----

%\begin{lemma}
%    The degree to which $\trace[\hat{f}(\rho)\rho(\sigma^{-1})]$ is correlated with $\trace[\rho(\sigma)]$ is the degree to which "an irrep concentrates on a subgroup implies that the character of certain conjugacy classes with respect to that irrep is highest for conjugacy classes that are most often found in that subgroup"
%\end{lemma}

%\begin{proof}
%    $f_{H}(\sigma) = 1$ if $\sigma \in H$ and 0 otherwise. $\hat{f}(\rho_{(n)}) = |H|$ and the Fourier transform of every irrep besides $\rho_{(n)}$ (the trivial irrep) and $\rho$ are 0. Then, by Fourier inversion:
%    $$ f(\sigma) = \frac{d_{\rho}}{|G|} \trace [\hat{f}(\rho)\rho(\sigma^{-1})] + \frac{|H|}{|G|}$$ 
%    $f_{H}(\sigma)$ is either equal to 0 or 1, so solving for the trace, when $\sigma \in H$, then $\trace [\hat{f}(\rho)\rho(\sigma^{-1})] = \frac{1}{d{\rho}}(|G| - |H|)$ and when $\sigma \notin H$ then $\trace[\hat{f}(\rho)\rho(\sigma^{-1})] = -\frac{1}{d_{\rho}}|H|$ 
%\end{proof}

\clearpage
\section{Extra Graphs}
\label{appendix:graphs}
%\subsection{Coset Circuit Logits}

%\begin{figure}
%    \includegraphics[width=8cm]{figures/logit_constructive_interference_boxplot.pdf}
%    \caption{The effect of combining coset circuits. These show the contributions of one, two, and all four $S_4$ coset circuits to the logits of each permutation when the target permutation is \emph{not} in $S_4$. When the target \emph{is} in $S_4$ then none of these neurons fire and their contributions to the logits are each identically zero.}
%    \label{fig:logits-constructive}
%\end{figure}

\subsection{Distribution over Subgroups and Cosets}
\label{sec:coset-circuit-distribution}

\begin{figure*}[hb]
\centering
    \centering   
    \begin{subfigure}[128 Models trained on $S_5$]{
        \includegraphics[width=0.45\linewidth]{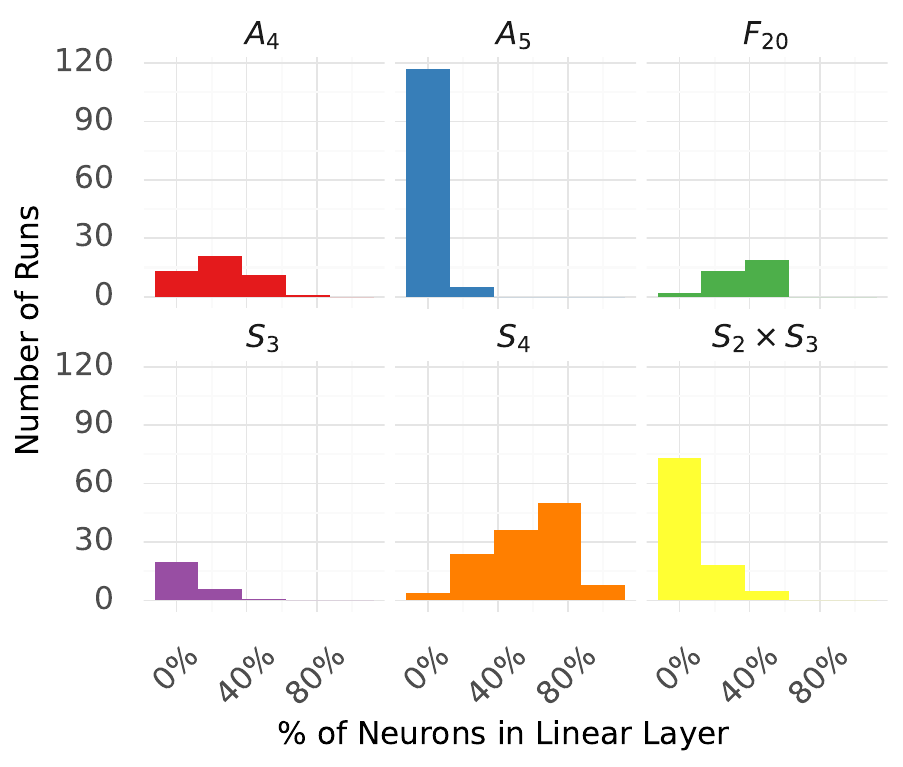}
        \label{fig:s5_small_coset_dist}}
    \end{subfigure}
    \hfill
    \begin{subfigure}[100 Models trained on $S_6$]{
        \includegraphics[width=0.45\linewidth]{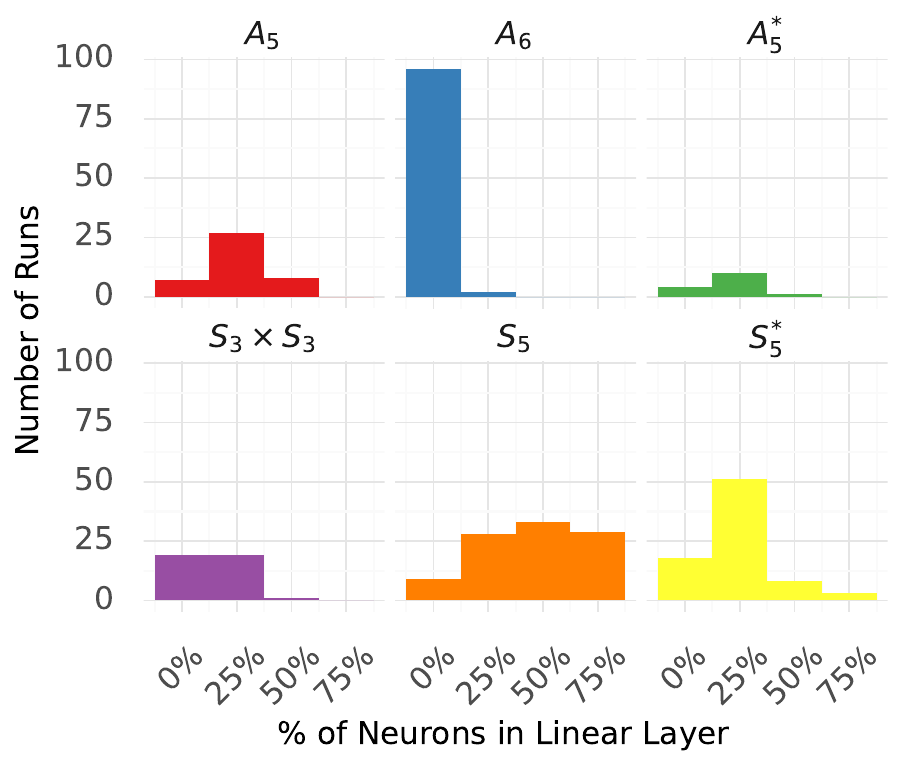}
        \label{fig:s6_coset_distribution}}
    \end{subfigure}
    \caption{Distribution of coset circuits for models trained on $S_5$ and $S_6$ with different initial seeds. Every model has a few sign circuit neurons that correspond to $A_n < S_n$, but the model cannot completely solve the task with only the sign circuit, so there are never more than a few. Every other subgroup could, with enough neurons, be used to completely solve the the multiplication, but in general if a model primarily uses a single subgroup it is $S_{n-1}$ (in the main body of the paper we refer to these subgroups as $H_i$, for the element $i \in [n]$ that is fixed). Every model has at least a few $S_{n-1}$ neurons. Many models use a mix of subgroups and there is often a ``long tail'' of a subgroup being represented by only one or two neurons. The subgroups marked with asterisks, $A_{5}^{*}$ and $S_{5}^{*}$, correspond to the ``exceptional'' subgroups of $S_6$, which come from an outer automorphism that only $S_6$ has \citep{s6_outer_automorphism}. These subgroups are isomorphic to $S_5$ and $A_5$, but not conjugate to the subgroups that come from fixing an element in $\{1, \dots, 6\}.$}
    \label{fig:s6-coset-distribution}
    
\end{figure*}
\clearpage

\subsection{Other Examples of Coset Circuits Forming}

\begin{figure*}[h!]
\centering
    \centering   
    \begin{subfigure}[$S_3 \! \times \! S_2$ Left Permutations]{
        \includegraphics[width=0.7\linewidth]{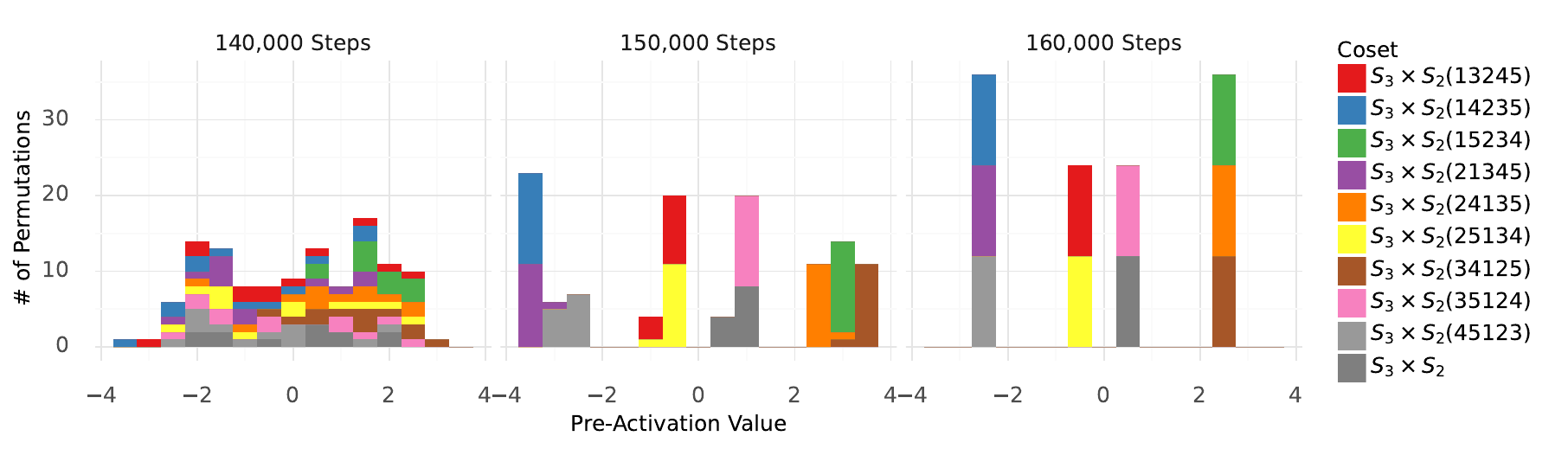}
        \label{fig:s3xs2_leftneuron}}
    \end{subfigure}
    \begin{subfigure}[$S_3 \! \times \! S_2$ Right Permutations]{
        \includegraphics[width=0.7\linewidth]{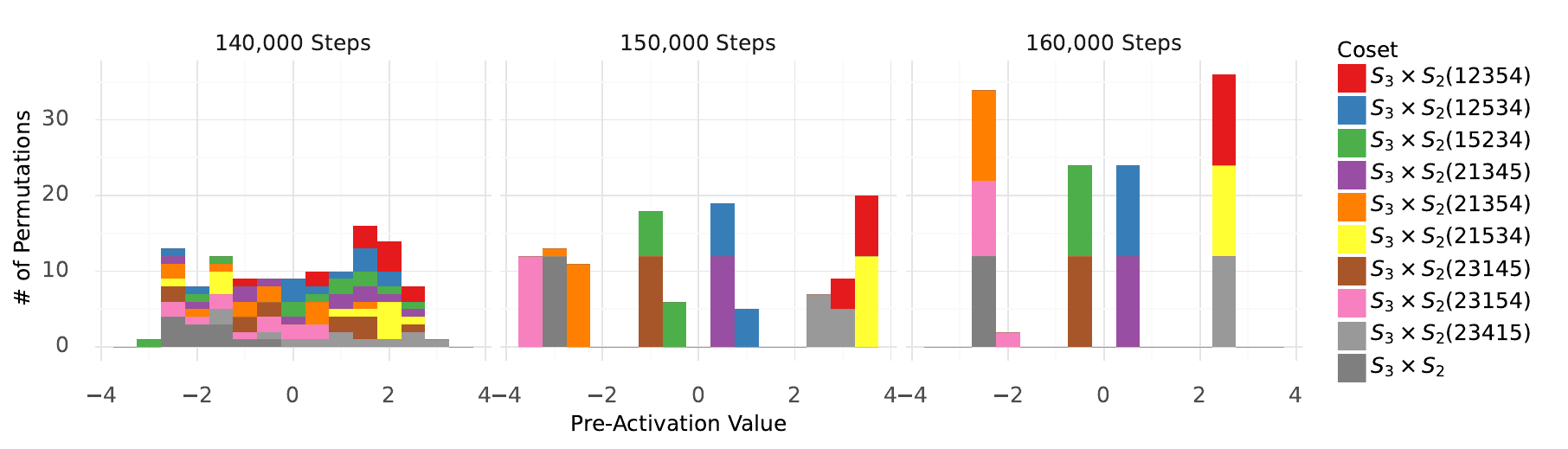}
        \label{fig:s3xs2_right_neuron}}
    \end{subfigure}
    \caption{The formation of an $S_3 \! \times \! S_2$ neuron.}    
\end{figure*}

\begin{figure*}[h!]
\centering
    \centering   
    \begin{subfigure}[$A_4$ Left Permutations]{
        \includegraphics[width=0.7\linewidth]{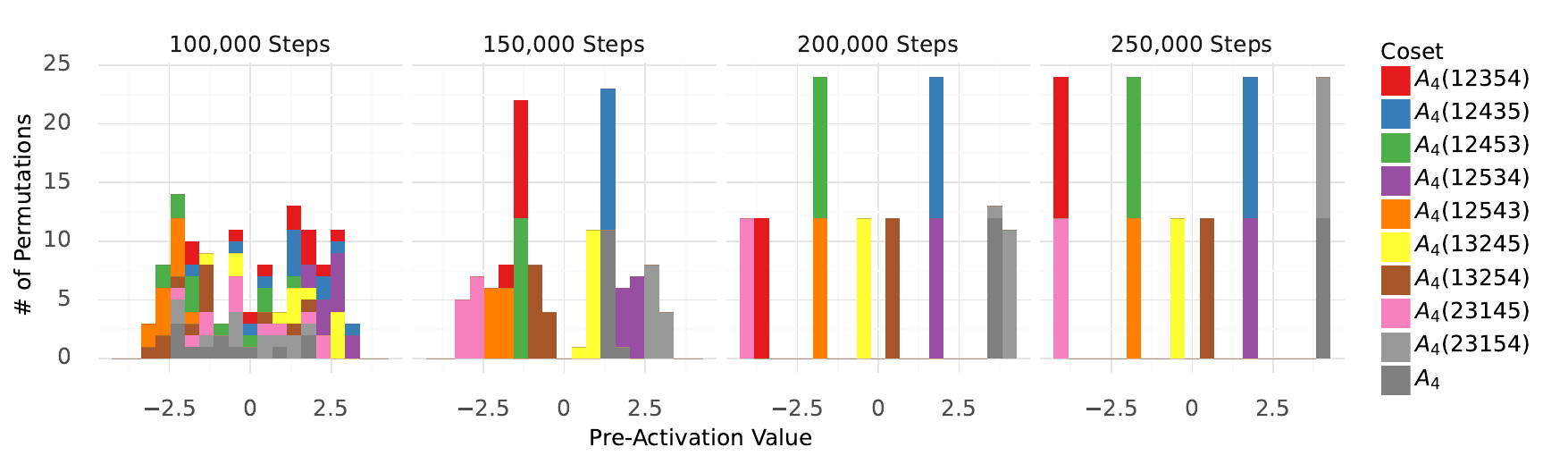}
        \label{fig:a4_leftneuron}}
    \end{subfigure}
    \begin{subfigure}[$A_4$ Right Permutations]{
        \includegraphics[width=0.7\linewidth]{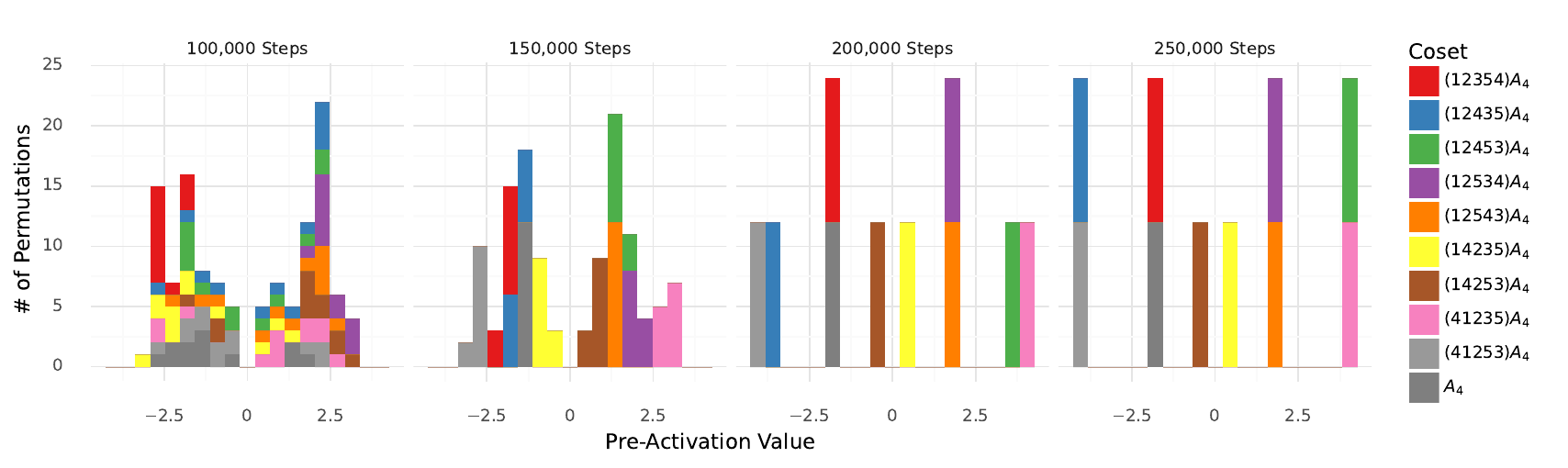}
        \label{fig:a4_right_neuron}}
    \end{subfigure}
    \caption{The formation of an $A_4$ neuron.}    
\end{figure*}

\clearpage

\section{Irreducible Representations}
\label{appendix:irrep}

\subsection{Symmetric Group $S_5$}

For a subgroup $H\le G$, we can investigate the Fourier transform of the indicator function $1_H$ by looking at its evaluation at each irrep. Concretely, we first center the indicator function by defining 
\[
f(g) = 
\begin{cases}
-\dfrac{|H|}{|G|}, ~g\not\in H\\
1-\dfrac{|H|}{|G|}, ~g\in H.
\end{cases}
\]
By doing so, $\hat{f}$ evaluates to $0$ on the trivial representation of $G$.

Given an irrep $\rho\in\irr(G)$, we first denote the value of the Fourier transform of $f$ at $\rho$ by $\hat{f}|_\rho$. The \emph{contribution} of $\rho$ to $\hat{f}$ is defined by the following:

\[
\dfrac{\| \hat{f}|_\rho \|^2}{\sum\limits_{\delta\in\irr(G)} \| \hat{f}|_\delta \|^2}
\]
Here, we list all the conjugacy classes of subgroups of $S_5$ and how each irrep of $S_5$ contributes to their centered indicator function. We center the indicator function to remove the contribution of the trivial irrep, which is only based on the index of the subgroup. This step makes the contributions comparable. In the first column, we show the homomorphism type of each subgroup. Recall that two groups $G, G'$ are \textit{homomorphic} if there exists a function $f: G \rightarrow G'$ such that for all $g,\;h \in G$, $f(gh) = f(g)f(h)$. Every group within a conjugacy class is a homomorphic, with the homomorphism of two subgroups $H, \; H'$ of $G$ given by conjugation by an element of $g \in G$, $h \mapsto ghg^{-1}$. Two conjugacy classes of subgroups, however, may be \textit{homomorphic} as groups, but no homomorphism can be given as conjugation by an element of $G$. Different conjugacy classes of subgroups that are homomorphic are distinguished in the second column by an example set of generators.
In the list:
\begin{itemize}
    \item $C_n$ means cyclic groups of order $n$.
    \item $S_n$ means the symmetric group of $n$ elements.
    \item $A_n$ means the alternating group of $n$ elements,the subgroup of $S_n$ consisting of even permutations. Recall than an ``even'' permutation is one that consists of an even number of transpositions.
    \item $D_{2n}$ means the $n$-gon dihedral group of order $2n$ (the symmetric group of regular polyhedron with $n$ edges).
    \item $F_{20}$ means the Frobenius group of order $20$, isomorphic to $C_4 \ltimes C_5$ \citep{dummit_abstract_2003}.
\end{itemize}
\begin{table}[htp]
\centering
\begin{tabular}{ccccccccc}
\toprule
Isomorphism type & Generators &  Size &  $(4, 1)$ & $(3, 2)$ & $(3, 1^2)$  & $(2^2, 1)$ &  $(2, 1^3)$ &  $(1^5)$\\ 
\midrule
$C_{2}$ & $\langle (1 2) \rangle$& 2  & 20.3\% & 25.4\% & 30.5\% & 17\% & 6.8\% & - \\
\hline
$C_{2}$ & $\langle (1 2)(3 4) \rangle$ & 2 & 13.6\% & 25.4\% & 20.3\% & 25.4\% & 13.6 & 1.7\%\\
\midrule
$C_{3}$ & $\langle (1 2 3) \rangle$ & 3 & 20.1\% & 12.8\% & 30.8\% & 12.8\% & 20.5\% & 2.6\% \\
\midrule

$C_{4}$  & $\langle (1 2 3 4) \rangle$ & 4 & 13.6\% & 25.4\% & 20.3\%  & 25.4\% & 13.6\% & 1.7\% \\
\midrule

$C_{2} \times C_{2}$ &  $\langle (1 2), (3 4) \rangle$ & 4 & 27.6\% & 34.5\% & 20.7\% & 17.2\% & - & - \\

\midrule 

$C_{2} \times C_{2}$  & $\langle (1 2)(3 4), (1 3)(2 4) \rangle$  & 4 & 13.8\% & 34.5\% & - & 34.5\% & 13.8\% & 3.5\% \\
\midrule

$C_{5}$  & $\langle (1 2 3 4 5) \rangle$ & 5 & - & 21.7\% & 52.2\% & 21.7\% & - &  4.4\% \\
\midrule

$C_{6}$  &$\langle (1 2 3), (4 5) \rangle$ & 6 & 21.1\% & 26.3\% & 31.6\% & - & 21.1\% & - \\
\midrule

$S_{3}$  & $\langle (1 2 3), (1 2) \rangle $ & 6  & 42.1\% & 26.3\% & 31.6\% & - & - & - \\
\midrule
$S_{3}$ \footnote{Referred to as "twisted" $S_3$ in plots.}  & $\langle (1 2 3), (1 2)(4 5) \rangle $& 6  & 21.1\% & 26.3\% & - & 26.3\% & 21.1\% & 5.3\% \\
\midrule

$D_{8}$  &$\langle (1 2 3 4), (1 3) \rangle $ & 8 & 28.6\%& 35.7\% & - & 35.7\% & - & - \\
\midrule

$D_{10}$  &$\langle (1 2 3 4 5), (2 5)(3 4) \rangle $ & 10 & - & 45.5\% & - & 45.5\% & - & 1\%\\
\midrule

$S_{3}\!\times\!S_{2}$  & $\langle (1 2 3), (1 2), (4 5) \rangle $ & 12 & 55.6\% & 44.4\% & - & - & - & - \\
\midrule

$A_{4}$  &$\langle (1 2)(3 4), (1 2 3) \rangle $ & 12 & 44.4\% & - & - & - & 44.4\% & 11.2\%  \\
\midrule

$F_{20}$  & $\langle (1 2 3 4 5), (2 3 5 4) \rangle $& 20 & - & - & - & 100\%  & - & - \\
\midrule

$S_{4}$  & $\langle (1 2 3 4 5), (1 2) \rangle $& 24 & 100\% & - & - & - & - & - \\
\midrule

$A_{5}$  & $\langle (1 2 3 4 5), (1 2 3) \rangle$& 60 & - & - & - & - & -  & 100\% \\

\bottomrule 
\end{tabular}
\caption{Subgroups of $S_5$ and the contribution of each irrep to their centered indicator function.}\label{tab:s5_subgroups}
\end{table}

%%%%%%%%%%%%%%%%%%%%%%%%%%%%%%%%%%%%%%%%%%%%%%%%%%%%%%%%%%%%%%%%%%%%%%%%%%%%%%%
%%%%%%%%%%%%%%%%%%%%%%%%%%%%%%%%%%%%%%%%%%%%%%%%%%%%%%%%%%%%%%%%%%%%%%%%%%%%%%%

\end{document}